\documentclass[journal]{IEEEtran}
\usepackage[T1]{fontenc}
\usepackage{amsmath,amssymb,booktabs,graphicx,url,array}
\usepackage[hidelinks]{hyperref}
\hypersetup{pdftitle={TableVision: Recovering Verified Structure from Historical Tables},pdfauthor={Todd Y. Zhou and James M. Snyder, Jr.},pdfsubject={Arithmetic-gated structure recovery for historical election tables and a theory of the sum check},pdfkeywords={Table structure recovery, historical documents, optical character recognition, data validation, dynamic programming, selective release, group testing, computational social science}}
\newcommand{\Exact}{\textsc{Exact}}
\newcommand{\Source}{\textsc{SourceDiscrepancy}}
\newcommand{\Hold}{\textsc{Withheld}}
\newcommand{\Unverified}{\textsc{Unverified}}
\newcommand{\mk}[1]{\raisebox{0.1ex}{\textcircled{\ifnum#1<10 \scriptsize\else\tiny\fi #1}}}
\newtheorem{theorem}{Theorem}
\newtheorem{proposition}{Proposition}
\newtheorem{corollary}{Corollary}
\newtheorem{lemma}{Lemma}
\begin{document}
\title{TableVision: Recovering Verified Structure\protect\\from Historical Tables}
\author{Todd~Y.~Zhou and James~M.~Snyder,~Jr.%
\thanks{T. Y. Zhou is with Harvard University, Cambridge, MA 02138 USA.}%
\thanks{J. M. Snyder, Jr. is with the Department of Government, Harvard University, Cambridge, MA 02138 USA.}}
\maketitle
\begin{abstract}
Gaps, skips, and mistakes are the normal state of a historical register: rows go missing, tables cross page breaks, columns slip under wrong names, and column sums disagree with the corresponding row entries. Readers read digits correctly but lose the table structure. On seven clean panels of election returns, the best open-source reader reproduces 99.0\% of the digits but files only 83.7\% of the figures in the correct cells and gets 11 of 30 columns correct. TableVision treats table structure recovery as the problem. On three degraded 1888 pages, Amazon Textract files only 1 of 1494 figures correctly; TableVision correctly files 99.2\%, and all 51 columns its exact-total check releases are correct. A dynamic program cuts a volume into blocks by how many columns sum to their last line; witnesses read the cells; the gate classifies each column as exactly reconciling, flagged as faithful to a page whose arithmetic is wrong, or withheld. A theory of the sum check says which errors it cannot see and why its optimal block is about $\theta^{-1/2}$ rows at defect rate $\theta$. On 2201 columns of Wisconsin electoral data, the gate accepts 2042, flags 118, and withholds 41; failure rises with column length $R$ as $1-(1-\theta)^R$ at 1.1 defects per 1000 cells in general elections; official returns confirm 933 of 934 testable accepted sums and all 35 testable flagged totals; and the deployment added 491\,359 gated precinct rows to a 9.3-million-row dataset. One implication is that the design of the verification block is critical.
\end{abstract}
\section{Introduction}
\label{sec:intro}
\IEEEPARstart{O}{n} pages 126--127 of the 1894 Iowa Official Register, Delaware County's returns run across a table 25 columns wide, printed sideways on portrait pages in type whose figures are nine pixels tall once scanned. An earlier extraction of those pages produced a county that could not be released, and the reason was not a misread digit: the extraction had classified the line for Union township as a heading, and the county's columns came up short by exactly that township's vote. The digits were correct. The table was wrong.

\begin{table}[t]
\caption{Digits Read Versus Table Structure Kept, by Readers Alone and With TableVision}
\label{tab:bench}
\centering\footnotesize
\setlength{\tabcolsep}{0pt}
\renewcommand{\arraystretch}{1.25}
\newcolumntype{C}[1]{>{\centering\arraybackslash}p{#1}}
\newcolumntype{R}[1]{>{\raggedleft\arraybackslash}p{#1}}
\newcolumntype{L}[1]{>{\raggedright\arraybackslash}p{#1}}
\newcommand{\tv}[1]{\textbf{#1}}
\begin{tabular}{@{}p{72pt}C{42pt}R{28pt}@{\enspace$\rightarrow$\enspace}L{25pt}R{23pt}@{\enspace$\rightarrow$\enspace}L{27pt}@{}}
\toprule
& \% of Digits & \multicolumn{2}{c}{\% of Numbers} & \multicolumn{2}{c}{County Columns}\\
Reader & Correct & \multicolumn{2}{c}{in Correct Cell} & \multicolumn{2}{c}{100\% Correct}\\
\midrule
\multicolumn{6}{@{}l}{\emph{Clean pages: 7 panels, 1214 numbers, 30 county columns}}\\[-1pt]
\multicolumn{6}{@{}l}{\emph{Reader's own table $\rightarrow$ the same words in TableVision's structure}}\\
Tesseract & 90.9 & 69.0 & \tv{83.4} & 3 & \tv{9}\\
EasyOCR & 77.3 & 41.6 & \tv{60.5} & 0 & \tv{1}\\
PaddleOCR & 98.1 & 81.4 & \tv{96.5} & 9 & \tv{14}\\
docTR & 99.0 & 83.7 & \tv{97.9} & 11 & \tv{24}\\
\addlinespace[5pt]
\multicolumn{6}{@{}l}{\emph{Hard pages: 3 pages of 1888 returns, 1494 numbers, 77 county columns}}\\
Tesseract & 67.4 & \multicolumn{2}{c}{15.0} & \multicolumn{2}{c}{0}\\
EasyOCR & 29.9 & \multicolumn{2}{c}{\phantom{0}3.5} & \multicolumn{2}{c}{0}\\
PaddleOCR & 60.8 & \multicolumn{2}{c}{13.9} & \multicolumn{2}{c}{0}\\
docTR & 71.0 & \multicolumn{2}{c}{19.0} & \multicolumn{2}{c}{0}\\
Amazon Textract & 76.7 & \multicolumn{2}{c}{\phantom{0}0.1} & \multicolumn{2}{c}{0}\\
\tv{TableVision} & \tv{99.7} & \multicolumn{2}{c}{\tv{99.2}} & \multicolumn{2}{c}{\tv{67}}\\
\bottomrule
\end{tabular}
\par\smallskip
\parbox{\columnwidth}{\footnotesize \emph{Clean}: seven hand-cropped panels from three volumes; each open-source reader's own table, then the same words filed by a stand-in for TableVision's column comb, both at the configuration at which the reader alone files the most cells. \emph{Hard}: three continuation pages of the 1888 Iowa returns, one-bit scans at 118 pixels per inch with no row labels (Section~\ref{sec:human}); each reader's own output, and one stored pass of TableVision's reader, without the second witness or review (Section~\ref{sec:witness}). Readers: Tesseract 5.3, EasyOCR 1.7, PaddleOCR PP-OCRv4, docTR 1.1. \emph{\% of Digits Correct}: reference digits reproduced anywhere on the page. \emph{\% of Numbers in Correct Cell}: read correctly and in their own row and column. \emph{County Columns 100\% Correct}: every cell and the printed total correct. On four clean keyed pages Textract places 99.9\% of the numbers correctly, as does the better of two typists, and on the hard pages TableVision's comb lifts the open-source words to 13--66\% (Tables~S2 and~S4).}
\end{table}

Imperfection is the rule in a historical register. A register has \emph{gaps}: cells destroyed in the printing or the scan. It invites \emph{skips}: a county that begins and ends mid-page without repeating its header, a parent place name left on the previous page, a printed row easily taken for a heading. It has \emph{mistakes}: column totals that do not equal the sum of the corresponding rows, a broken glyph that a human and a machine both read as another digit (Fig.~\ref{fig:hard}). And when a column grid slides one window over, it files one candidate's figures under another's name with every sum intact. None can be seen from one cell; Table~\ref{tab:bench} measures what an OCR framing misses, on clean pages and on hard ones. Given the same clean panels from three volumes, open-source recognizers~\cite{tesseract,easyocr,ppocr,doctr} reproduce up to 99.0\% of the printed digits, yet in the tables they return, at most 83.7\% of the figures land in their own row and column and at most 11 of 30 county columns are correct, and the table extractors img2table and SLANet-plus~\cite{img2table,ppstructure} file fewer than half the figures in place (Table~S2 of the supplement). A commercial table reader, Amazon Textract~\cite{textract}, does far better on clean pages, misreading one of 1356 keyed figures, but that one is the figure a typist also misread. On hard pages the commercial reader fails worst. The 1888 Iowa returns continue each county's rows on a facing page that prints no row labels, scanned in one bit at 118 pixels per inch; Textract read 54 of that volume's 58 such pages sideways, and on the three in Table~\ref{tab:bench} it files 1 of 1494 figures in its cell, against 4--19\% for the open-source readers' own lines. A recognizer, learned or human, sees one cell and returns one value; given a column with a missing row it returns a complete-looking column, and a generative model asked to complete the table invents figures that are fluent rather than attested~\cite{ji}. The errors that matter are silent and structural: they concern which rows belong to which county under which name.

So we treat the recovery of a table's structure as the problem. A column's residual splits into a source term, a scope term, and a reading term (Eq.~\eqref{eq:decomp}); a recognizer, however good, removes only the last, and on legible pages the last is already small, so the errors that remain live in the other two, which are properties of the table and not of any cell. The table carries its own check, a total line that must equal the sum of the rows above it, and TableVision uses that arithmetic as evidence, never as a target to fit. It locates county blocks by a dynamic program over the whole volume that scores every candidate block by how many of its columns foot (sum exactly) to its last line, and reads cells with witnesses of different kinds. It releases a column only through an exact-total gate that no digit is altered to satisfy, so every column ends exactly reconciling, flagged because two readers judged the cells faithful to a page whose arithmetic does not close, or withheld. Filed by a column comb in TableVision's manner, the very words that docTR files at 83.7\% put 97.9\% of the figures and 24 of the 30 county columns in place (Table~\ref{tab:bench}), and of the 20 such columns that close on their panel, the exact-total check releases the 16 the comb got correct and withholds the 4 it got wrong (Table~S2 of the supplement). On the hard pages a single stored pass of the workflow's own reader files 99.2\% of the figures and 67 of the 77 county columns in place, and the check releases 51 columns, all exact: every column that closes on its page and that the pass read correctly, and none of the 8 closing columns it misread. A sum certifies what a reader nearly got correct; it cannot make a table out of noise: on the hard pages, where the open-source readers at their reported configurations file no county column correctly, all 23 columns the check releases from their output over every configuration run are wrong. A single sum is one check equation in the sense of error-detecting codes~\cite{hamming} and one pool in a group test~\cite{dorfman}; Section~\ref{sec:theory} says what it can and cannot see, and Section~\ref{sec:eval} tests each quantitative statement the test bed can reach, on 2201 county-by-candidate columns from five Wisconsin Blue Books, against evidence the workflow never saw.

The contributions are the formulation (Section~\ref{sec:problem}), in which only one term of a column's residual is a recognition error; the workflow, stated as constraints, predicates, and an exact segmentation recursion (Section~\ref{sec:workflow}); the theory of the sum check (Section~\ref{sec:theory}); and an evaluation against official returns, human transcription, and off-the-shelf readers rather than the workflow's own outputs (Section~\ref{sec:eval}). None of the tools is new. Three things are new. The formulation makes the book's arithmetic both the objective of segmentation and the certificate of release. The theory proves the cut a maximum-likelihood estimate under a stated model of the page, the gate's silent error at most second order in the defect rate, and the work of verification falling toward the information-theoretic limit as printed subtotals nest; its testable predictions (a one-parameter failure curve, about one defective total among 51 failures, a single-digit signature in the flagged residuals) the test bed bears out. The design conclusion, which no archival transcription we know of follows, is to choose the verification block by the defect rate among the sums a book prints, not by the county boundary. The precinct-level dataset is the subject of a companion data descriptor and serves here as the test bed.

TableVision was developed with a particular problem in mind, historical election data, but the problem of table structure is much more general. The same loss of structure meets anyone who scans a printed table: government statistics, census tables of demographic and socioeconomic information, education statistics such as school spending and teachers' salaries, historical chemical formulas, newspaper pages of stock returns, and the obscure price series nobody was ever going to key by hand. Massive amounts of economic, demographic, social, financial, and other records are stored in archives and libraries around the world; scanned, read, and put back into proper tables, they become data for social scientists and historians, as do tables of experimental results and physical measurements. TableVision applies to any scanned image of a printed table whose blocks close on a total, and its geometry, typing, constraints, and provenance trace apply to tables without totals, so its potential uses are many; Section~\ref{sec:disc} discusses this scope and gives one specific example.

\begin{figure*}[t]
\centering
\includegraphics[width=0.58\textwidth]{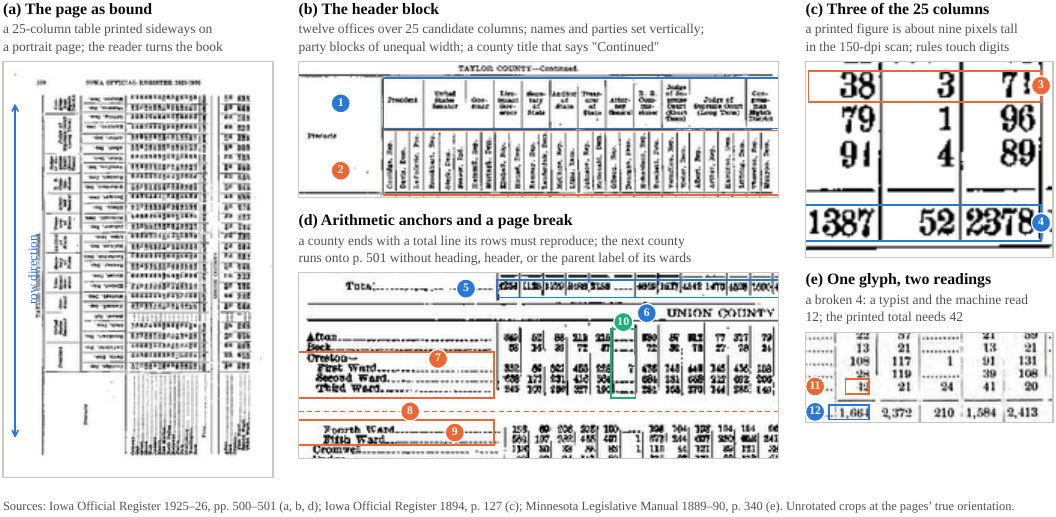}
\caption{What a per-cell reader cannot see. (a) Iowa Official Register 1925--26, p.~500, as bound: a 25-column table printed sideways. (b) Its header block: \mk{1} 12 office headers over \mk{2} 25 candidate columns in party blocks of unequal width, under a title that says ``Continued.'' (c) Three Delaware County columns, 1894 register: \mk{3} the Union township row taken for a heading and \mk{4} the total line each column must reproduce; a figure is about nine pixels tall. (d) The foot of p.~500 and head of p.~501: \mk{5} Taylor County's total line; \mk{6} the next county's heading; \mk{7} a parent label; \mk{8} a page break with no county heading, column header, or parent label, and \mk{9} wards continuing overleaf under that label; \mk{10} leader dots that mean blank, not zero. (e) Minnesota 1889--90, p.~340: \mk{11} a broken 42 that a typist and the machine both read as 12, and \mk{12} the printed total that needs 42.}
\label{fig:hard}
\end{figure*}
\section{Related Work}
Table recognition has its own benchmarks and metrics: PubTabNet and tree-edit-distance similarity (TEDS)~\cite{pubtabnet}, TableFormer~\cite{tableformer}, PubTables-1M and GriTS~\cite{pubtables,grits}; end-to-end models such as Donut~\cite{donut} emit the table, or the page, as a token sequence, and a vision-language model prompted to transcribe a page is the same kind of reader with the same page-level score. All of these score one page against a reference rendering of it, and a character score, even one independent of reading order~\cite{fca}, says nothing about where a figure is filed. None asks whether a released number belongs to the county it is filed under, whether that county is complete, or whether a correctly identified total supports it: a model that returns a well-formed table for a page with a missing row scores well and is wrong.

Using a table's own arithmetic to check it is an old idea. Clausner et al. validate values extracted from historical statistical tables against their totals and route inconsistencies to targeted processing~\cite{clausner}; TableVision also uses the arithmetic before any column is verified, as the objective that locates the county blocks (Section~\ref{sec:segment}), refuses to release a failed column without a second reader's attestation, and states what the check certifies as theorems tested against official returns. Dewulf et al. decode under symbolic constraints and fine-tune with them~\cite{dewulf}; TableVision applies the constraint only afterward, as a gate on release rather than a bias on decoding. The gate is one check equation over the integers, and the questions asked of an error-detecting code~\cite{hamming}, which patterns pass undetected and what protection the check gives, are the ones Section~\ref{sec:theory} answers; verification against nested subtotals is hierarchical group testing in Dorfman's sense~\cite{dorfman}; and locating blocks by a global optimum over the row sequence is dynamic programming~\cite{bellman} with the book's arithmetic as objective.

Abstention runs from Chow's error/reject tradeoff~\cite{chow} to selective prediction with a learned reject option~\cite{selectivenet}; TableVision abstains at the county-by-candidate column on document evidence rather than a confidence score. Constraint-based data cleaning repairs a relation by the minimal change that satisfies its integrity constraints~\cite{ilyas}; we refuse the repair, because a value chosen to satisfy a sum inherits every other error in its column (Proposition~\ref{prop:repair}), and release instead what provenance~\cite{buneman} and data debugging~\cite{checkcell} name: the trace.\section{Problem Formulation}
\label{sec:problem}
\subsection{Input, Output, Constraint, and Objective}
A \emph{volume table} is the sequence of printed lines $\ell_1,\dots,\ell_M$ of one contest's tables in one volume, in reading order (Section~\ref{sec:geom}). Each line carries a label string; a type $\tau(\ell)\in\mathcal T$ assigned by Section~\ref{sec:struct}, where $\mathcal T$ comprises unit, subtotal, total, heading, absentee, comparative, footnote, and other lines; and, for each of the $C$ candidate columns fixed by the table header, a reading $x_{\ell c}\in\mathbb{Z}_{\ge0}\cup\{\bot\}$, the parse of the cell string to a nonnegative integer, $\bot$ when the cell is illegible or does not parse. A cell that prints no figure, such as a run of leader dots, is blank and adds nothing to a sum. The number of counties $K$ is known (99 in Iowa, 71 in Wisconsin), and printed county totals close the county blocks. The output is (i) a cut of the constituent lines into $K$ county blocks; (ii) for every block $k$ and column $c$ a \emph{verification column} with a state in $\{\Exact,\allowbreak\Source,\allowbreak\Hold,\allowbreak\Unverified\}$ (Section~\ref{sec:states}); and (iii) one record per (constituent line, column) of every released column, with the column's state and a provenance trace to its cell where one exists. The constraint is that every released value is a reading of the page, one of the witness readings of that cell (Section~\ref{sec:witness}), never a value imputed from a total, a neighboring column, or a model prior; the objective is to release as many cells as possible subject to that constraint and to a certificate on every released column, at a silent error rate that Section~\ref{sec:theory} bounds.

\subsection{Verification Columns and the Residual}
The unit of everything that follows is the verification column, the vector of constituent-line votes for one candidate in one county, contest, and election. A block that the cut assigns $\hat R$ constituent lines and one total line has readings $\hat V\in(\mathbb{Z}_{\ge0}\cup\{\bot\})^{\hat R\times C}$ and $\hat t_c\in\mathbb{Z}_{\ge0}\cup\{\bot\}$. Column $c$ is \emph{complete}, $\mathrm{Cmp}_c\equiv[\hat t_c\ne\bot]\wedge[\forall i:\hat V_{ic}\ne\bot]$, when every entry and the total are available, and then its residual is
\begin{equation}
 d_c=\sum_{i=1}^{\hat R}\hat V_{ic}-\hat t_c.
 \label{eq:residual}
\end{equation}
The workflow demands $d_c=0$ exactly: counts are integers, and a tolerance would only hide the small errors that still change an observation.

\subsection{Gaps, Skips, and Mistakes}
Let $\mathcal L$ be the set of printed lines of the source, $\mathcal U_k\subseteq\mathcal L$ the $R$ constituent lines that county $k$ prints, $v_{\ell c}$ the figure line $\ell$ prints in column $c$ ($v_{\ell c}=0$ for a line that prints no figure there), and $t_c$ the printed total. The extraction map $\rho:\{1,\dots,\hat R\}\to\mathcal L\cup\{\varnothing\}$ sends extracted row $i$ to the printed line it was cut from; a row assembled from no printed line, such as a subheading that acquired a reading from the column rule, has $\rho(i)=\varnothing$ with $v_{\varnothing c}:=\hat V_{ic}$, so that it counts as a non-constituent line included in the block. Reading errors are $e_{ic}=\hat V_{ic}-v_{\rho(i)c}$ and $e_{tc}=\hat t_c-t_c$. Whenever $\mathrm{Cmp}_c$ holds,
\begin{align}
 d_c={}&\underbrace{\Big(\sum_{\ell\in\mathcal U_k} v_{\ell c}-t_c\Big)}_{b_c\ \text{(source)}}+\underbrace{\Big(\sum_{i=1}^{\hat R} v_{\rho(i)c}-\sum_{\ell\in\mathcal U_k} v_{\ell c}\Big)}_{\sigma_c\ \text{(scope)}}\notag\\
 &+\sum_{i=1}^{\hat R} e_{ic}-e_{tc},
 \label{eq:decomp}
\end{align}
the identity Section~\ref{sec:theory} builds on. The three kinds of defect named in the introduction are the three kinds of term. A \emph{gap} is a cell at $\bot$: $\mathrm{Cmp}_c$ fails and the residual does not exist. A \emph{skip} is a scope error $\sigma_c\ne0$: a constituent line the block lacks or repeats, a non-constituent line it includes, a county cut in the wrong place, with the printed digits intact. A \emph{mistake} is the book's, $b_c\ne0$, or the reader's, $e_{ic}\ne0$ or $e_{tc}\ne0$. Identity errors, a column filed under a neighboring candidate's name or a block under the wrong county, lie outside Eq.~\eqref{eq:decomp}, leave $d_c=0$, and are the business of the constraints of Section~\ref{sec:struct}. Only $e$ is a recognition error, and it is the only term a better recognizer removes.

\subsection{Release States}
\label{sec:states}
Three predicates on evidence decide a column's state: $\mathrm{Cmp}_c$; $\mathrm{Cred}_c$, that the total reading is corroborated (Section~\ref{sec:gate}); and $\mathrm{Agr}_c$, that every cell has been re-read from the page by a second reader, at magnification where disputed, and confirmed (Section~\ref{sec:adjud}). Then
\begin{align}
 \Exact &\iff \underbrace{\mathrm{Cmp}_c}_{\text{(G1)}}\wedge\underbrace{[d_c=0]}_{\text{(G2)}}\wedge\underbrace{\mathrm{Cred}_c}_{\text{(G3)}},\notag\\
 \Source &\iff \mathrm{Cmp}_c\wedge[d_c\ne0]\wedge\mathrm{Agr}_c,\label{eq:states}\\
 \Hold &\iff \text{neither},\notag
\end{align}
and \Unverified\ is the state of material with no printed total to check against (none in the test bed). The release is $\Exact\cup\Source$. An \Exact\ column certifies exactly one thing: its cells sum to a printed county total that $\mathrm{Cred}_c$ corroborates. A \Source\ column certifies that two readers confirmed every cell against a page whose own arithmetic fails by the logged residual; which of its $R+1$ printed numbers is wrong is then a matter of odds (Proposition~\ref{prop:which}). A \Hold\ column certifies nothing and is absent, not zero; its printed total is not released on its own, because a total is a reading that only its rows can check.
\section{The TableVision Workflow}
\label{sec:workflow}
Page geometry (Section~\ref{sec:geom}) turns each page into typed lines (Section~\ref{sec:struct}) with cell slices, which a first reading pass values as $x^{(1)}_{\ell c}$; segmentation (Section~\ref{sec:segment}) runs once per volume table on that sequence and returns the $K$ blocks; the remaining witnesses, the gate, adjudication, and the diagnostics (Sections~\ref{sec:witness}--\ref{sec:diag}) act on the $K\times C$ verification columns, and a block whose missing line is recovered is re-gated. What follows is the procedure, not a particular recognizer.

\subsection{Page Geometry at Native Resolution}
\label{sec:geom}
Each page is taken as the bitonal scan $I$ embedded in the PDF. The Iowa registers we developed on are 150-dpi scans in which a printed figure is about nine pixels tall; enlarging such a glyph adds no information while inventing shape, so first-pass cells are read at native resolution, and a re-read ``at magnification'' is nearest-neighbor upsampling of the same pixels. Orientation is a declared input, and skew is removed before any band is cut: the lag $\Delta$ that best aligns the ink profiles of the left and right halves of the table region, whose centers are a distance $W$ apart, gives $\hat\alpha=\arctan(\Delta/W)$; in the sideways registers $\Delta$ is about 20 pixels, a little over one degree and more than half a row pitch (supplement, Table~S3, row~1). The horizontal profile $P(y)$ of a band then gives row centers $y_1<\dots<y_n$ as the local maxima of its smoothed version at least half a pitch apart, the pitch $s$ taken first from the autocorrelation of $P$ and refined to $\operatorname{med}_r(y_{r+1}-y_r)$, and the vertical profile gives the column rules as maximal runs of inked image columns, the intervals between them forming the \emph{column comb} $\mathcal W_1,\dots,\mathcal W_C$. Two guards protect the comb: (a) the number of windows equals the header's column count $C$, and (b) no rule run is wider than a guard width $g$, deployed at 26 pixels; a violation re-measures the comb, and the page is held until both pass.

\subsection{County-Block Segmentation by Arithmetic}
\label{sec:segment}
The constituent and total lines of a volume table form one sequence of $N$ typed lines that has to be cut into $K$ county blocks before any column can be summed; headings and footnotes are unreliable cues, since either can be missed. Let $x_{rc}\in\mathbb{Z}_{\ge0}\cup\{\bot\}$ be the first-pass reading $x^{(1)}_{rc}$ of line $r$ in column $c$, the superscript dropped in this subsection, $P_c(r)=\sum_{r'\le r}x_{r'c}$ the prefix sums with $\bot$ counted as zero, and $U_c(r)$ the number of $\bot$ among the first $r$ lines. For a candidate block $[i,j]$ whose last line $j$ is taken as its total line, the footing count is
\begin{equation}
 \phi(i,j)=\sum_{c=1}^{C}\big[x_{jc}\ne\bot,\ \Delta U_c=0,\ \Delta P_c=x_{jc}\big],
 \label{eq:foot}
\end{equation}
where $\Delta P_c=P_c(j-1)-P_c(i-1)$ and $\Delta U_c=U_c(j-1)-U_c(i-1)$ are the block's column sum and its count of unread cells. The block score is $S(i,j)=\phi(i,j)+\lambda_h h_i+\lambda_f f_j+\lambda_t\eta_j$, where $h_i=1$ if a heading line immediately precedes line $i$, $f_j=1$ if a footnote follows line $j$, $\eta_j=1$ if the label of line $j$ reads ``Total,'' and the nonnegative weights $\lambda$ are small against one footing column. With $K$ known and every block between two and $L$ lines long, the segmentation is the cut $0=j_0<j_1<\dots<j_K=N$ that maximizes $\sum_kS(j_{k-1}+1,j_k)$, found by
\begin{equation}
 F(k,j)=\max_{i}\big[F(k-1,i)+S(i+1,j)\big],
 \label{eq:dp}
\end{equation}
over $\max(0,j-L)\le i\le j-2$, with $F(0,0)=0$ and $F(k,j)=-\infty$ where no such cut exists, and the cut recovered by backtracking from $F(K,N)$.

\begin{lemma}[Exactness and complexity]
\label{lem:dp}
$F(k,j)$ is the maximum score of a cut of lines $1,\dots,j$ into $k$ blocks of lengths in $[2,L]$, so the recursion returns a global maximizer over all such $K$-cuts, in $O(NL(C+K))$ time and $O(NL+KN)$ space beyond the $O(NC)$ readings and their prefix sums (proof in the supplement).
\end{lemma}

The maximizer is the true segmentation when every county block, between two and $L$ lines long, foots in more columns than any block ending off a county's total line, and no other block ending on a county's total line foots in more columns than that county; it stays so for cue weights below $1/(3K)$, and the chance that some coincidental block foots widely is bounded (Theorem~4 of the supplement). Without the cue terms it is also the maximum-likelihood cut under an explicit model of the page, and with them a posterior mode (Theorem~\ref{thm:infer}). Columns with many zero entries threaten the first condition, and the conditions can be checked only from a row-level inventory of readings, which the audited release does not contain. A block that foots in no column is \emph{unanchored}: its position rests on cues alone, and all of its columns enter \Hold\ until the diagnostics of Section~\ref{sec:diag} explain it.

\subsection{Reading With Several Witnesses}
\label{sec:witness}
A cell has up to four witnesses, merged by a fixed rule. Two transcription passes read each column from identical images without seeing each other's output, giving $x^{(1)}_{ic}$ and $x^{(2)}_{ic}$; where they agree, $\hat V_{ic}=x^{(1)}_{ic}$, and where they disagree, or either is $\bot$, the cell goes to a magnified re-read $x^{(m)}_{ic}$ at five to eight times, shown with its contiguous neighbors on the printed grid, and $\hat V_{ic}=x^{(m)}_{ic}$, which is $\bot$ if the magnified glyph is itself illegible. Values are never averaged. A reader here is a vision-language-model pass (Claude-family models in this deployment) prompted to return the table as JSON with its printed total line, and a second pass adds evidence without making errors independent (Section~\ref{sec:human}).

Where the PDF carries a text layer, it serves as a third witness of a different kind: a numeric token is assigned to the nearest fitted row if it lies within half a pitch of it and to the comb window that contains its box, and is discarded otherwise. The layer is admitted only in columns it covers almost completely (on the 563-row Milwaukee County table of the 1940 volume, 557 cells in a three-digit column against 51 in a small one), and there every cell on which it and the readers differ is inspected by eye first, 167 of the 2815 cells in the five admitted columns. The layer slipped rows and misread digits: a witness, never an arbiter.

A fourth signal needs no second reading: the glyph distance $\zeta_{ic}=\min_{d}\|I_{ic}\oplus\Gamma_d\|_0$ of the cell image $I_{ic}$ to per-volume digit templates $\Gamma_d$ ranks a page's columns for review and enters no gate; in the 1940 volume one column had a mean distance of 10 against 3 to 7 for the other 10 and accounted for 10 of the 32 failing county columns outside Milwaukee. For a failed column the same model lists the single-digit substitutions that would close the residual (Lemma~\ref{lem:syndrome}), as places to look, never as values.

\subsection{Typing and Structural Constraints}
\label{sec:struct}
Every line is typed from its label and position: ``Total'' closing a block is a $\mathsf{total}$, a municipality line whose wards follow it or a ``Total'' inside a block a $\mathsf{subtotal}$, absentee and comparative lines go by their labels, a place label with figures is a $\mathsf{unit}$, and a label without figures a $\mathsf{heading}$ or $\mathsf{other}$. Every block must then satisfy six constraints. (S1) \emph{Width}: a $\mathsf{unit}$ line has exactly $C$ parsed cells; a line with more or fewer is flagged, not read. (S2) \emph{Hierarchy}: every ward line has a parent municipality label in scope, inherited from the nearest preceding parent label, across a page break if necessary. (S3) \emph{Continuation}: a page without a header inherits the previous page's header if and only if the volume, the contest title, $C$, and the comb positions (within one pitch) agree. (S4) \emph{Exclusion}: only $\mathsf{unit}$ lines are constituents of a sum. (S5) \emph{Identity}: column $c$ carries the header's $c$th entry under the guarded comb, and a comb that fails a guard carries no identity. (S6) \emph{Edge}: a window whose right boundary falls outside the page's image box is $\bot$ on that page. The constraints exist because plausible arithmetic is not valid arithmetic: a repeated page contains internally consistent county blocks, and a column shifted under an adjacent header keeps its integers, as when a 28-pixel double rule shifted a whole comb and filed one candidate's total under another's name, which the 26-pixel width guard exposed (S5; Table~S3, row~2).

\subsection{The Exact-Total Gate}
\label{sec:gate}
A column enters the release as \Exact\ only through the predicate $G_c\equiv\mathrm{Cmp}_c\wedge[d_c=0]\wedge\mathrm{Cred}_c$ of Eq.~\eqref{eq:states}. (G3) exists because an exact match against a misread total proves nothing, and under (G2) the cells cannot impeach the total, since nonnegative cells that sum to $\hat t_c$ never exceed it. $\mathrm{Cred}_c$ holds if a magnified re-read of the total cell returns $\hat t_c$, or if no re-read was requested because $\hat t_c$ passed the \emph{coherence screen}, lying within the range of the other read totals of its party block (the columns under one party's heading). The screen is a prior, not a reading: in the straight-ticket decades of these registers a party's candidates receive nearly identical county totals, so a total that stands apart from its block is the one to re-read, and the re-read is the evidence.

Two development episodes show the gate refusing what an optimizer would accept: a destroyed cell in Polk County, Iowa, 1916, whose 75 siblings sum to 12\,354 against a printed 12\,621, was left unread rather than written as 267. A shortcut that filled missing cells from a neighboring column of the same party cut one volume's exact columns from 771 to 534, since ticket splitting makes a neighbor a poor predictor, and none of its 2224 substitutions entered a release.

\subsection{Adjudication and Release}
\label{sec:adjud}
A column that fails $G_c$ enters a review whose steps are ordered by the odds of Proposition~\ref{prop:which} and by effort: the row-dominance test, a magnified re-read of the total, the omitted-row test of Section~\ref{sec:diag}, and only then a second reader's re-read of every cell. Row dominance comes first because it needs no new reading and is decisive: if $\max_i\hat V_{ic}>\hat t_c$, the column is certainly defective and the defect is either that cell's reading or the total, so once the cell is confirmed at magnification the total is the defective number, printed or read. In Winnebago County in the 1940 volume, Oshkosh alone casts 77 votes for a candidate whose printed county total is 64. A proposed correction is a magnified reading of the page and nothing else; the arithmetic does not choose it but decides only whether the corrected column is released, since a value chosen to make the column foot would void its certificate (Proposition~\ref{prop:repair}). A column whose every cell two readers have confirmed against the page, and whose residual survives, satisfies $\mathrm{Agr}_c$ and is released as \Source; a column with an unread cell, a disputed cell, or an unexplained residual stays \Hold. The repair log records the readers and the outcome of each investigation (Section~\ref{sec:review}); three applied cell corrections record a single reader.

\subsection{Diagnosing Structural Failures}
\label{sec:diag}
When a whole county fails, the pattern of residuals across its columns tells a missing line (every column short by that line's own figure, Proposition~\ref{prop:signature}) from scattered misreads (columns short by unrelated amounts). For every column with $\hat t_c\ne\bot$ let $\underline d_c=\sum_{i:\hat V_{ic}\ne\bot}\hat V_{ic}-\hat t_c$ be the residual of its readable cells, equal to $d_c$ when the column is complete. The single-line screens, over the columns with a read total, are
\begin{align}
 \mathrm{Omit}_1&\equiv[\forall c:\underline d_c\le0]\wedge[\exists c:\underline d_c<0]\wedge\mathrm{Gap},\notag\\
 \mathrm{Dup}_1&\equiv[\forall c:\underline d_c\ge0]\wedge[\exists c:\underline d_c>0],\label{eq:omit}
\end{align}
where $\mathrm{Gap}$ holds when two consecutive detected rows of the block sit more than $1.4$ pitches apart with a label line assembled between them. They are screens only: a candidate line is restored when it is a constituent line on the page and its magnified readings $u_c$ equal $-d_c$ in the complete columns that then foot, one exact equality per column; a column that still fails after the restoration has a second defect and stays \Hold. In the 1894 Iowa register four counties with nothing released were each missing one printed line that the row classifier had taken for a heading or never detected. In Delaware County the Union township line read off printed pp.~126--127 carries exactly the 38, 3, and 71 by which three columns were short, 10 of the county's 25 columns then foot, and the other 15 stay withheld on their own evidence.
\section{A Theory of the Sum Check}
\label{sec:theory}
The gate certifies one integer equation per column. This section says what that certificate is worth and what it cannot see, why the same arithmetic finds the county blocks, and how verification should be laid out; Section~\ref{sec:eval} tests the predictions the test bed can reach, and the supplement proves every result in full.

The names come from Section~\ref{sec:problem}, with $v_c\in\mathbb{Z}_{\ge0}^{R}$ the vector of printed figures over the county's $R$ constituent lines; Eq.~\eqref{eq:decomp} reads $d_c=b_c+\sigma_c+\sum_ie_{ic}-e_{tc}$ whenever $\mathrm{Cmp}_c$ holds.

\subsection{What One Sum Can See}
\label{sec:kernel}
Fix a correctly scoped column, one for which $\rho$ is a bijection of $\{1,\dots,\hat R\}$ onto $\mathcal U_k$ (so $\hat R=R$ and $\sigma_c=0$), with consistent printed arithmetic ($b_c=0$), and write the readings as $v_c+e$, $t_c+e_t$ with $(e,e_t)\in\mathbb{Z}^{R}\times\mathbb{Z}$.

\begin{proposition}[Kernel]
\label{prop:kernel}
Condition (G2) holds if and only if $\mathbf 1^{\mathsf T}e=e_t$. The undetectable patterns therefore form the hyperplane $H=\{(e,e_t):\mathbf 1^{\mathsf T}e=e_t\}$, a sublattice of rank $R$ in $\mathbb{Z}^{R+1}$: every pattern with exactly one nonzero component is detected, and every permutation of the entries of $v_c$ is undetectable.
\end{proposition}

\begin{corollary}[Second constraints]
\label{cor:subtotal}
Partition the lines into groups $g\in\mathcal G$ with separately printed subtotals, read with errors $e_{s_g}$, and keep the county total. A pattern $(e,(e_{s_g})_g,e_t)$ is undetectable if and only if $\mathbf 1_g^{\mathsf T}e=e_{s_g}$ for every $g$ and $\mathbf 1^{\mathsf T}e=e_t$, so the undetectable set has codimension $|\mathcal G|+1$ in $\mathbb{Z}^{R+|\mathcal G|+1}$ instead of one; with the subtotals read correctly, a permutation is detected unless it preserves every group sum, and a compensating pair unless both members share a group.
\end{corollary}

Typist~A's most frequent error, a displacement within the column, lies in $H$ whenever the displaced values stay inside the column (Section~\ref{sec:human}), which is why the sum check is composed with the constraints of Section~\ref{sec:struct} rather than trusted alone. Suppose now that two witnesses read the column as $x^{(1)}$ and $x^{(2)}$, disagreeing on $D=\{i:x^{(1)}_i\ne x^{(2)}_i\}$ with $\delta_i=x^{(2)}_i-x^{(1)}_i$.

\begin{proposition}[Arbitration by arithmetic]
\label{prop:arbitration}
Assume $\sigma_c=b_c=0$, that the total is read correctly ($e_t=0$), and that for every $i$ at least one witness has the printed figure. Then the true column is $x^{(1)}$ with $x^{(2)}$ substituted on some $T\subseteq D$, and $T$ satisfies $\sum_{i\in T}\delta_i=t_c-\mathbf 1^{\mathsf T}x^{(1)}$. The sum identifies $T$, and hence the column, if and only if that residual has a unique representation as a subset sum of $\{\delta_i\}_{i\in D}$. This holds for every residual when the $|\delta_i|$, sorted increasingly, are superincreasing (each exceeds the sum of all smaller ones), and in particular whenever $|D|=1$.
\end{proposition}

A shared error, $x^{(1)}_i=x^{(2)}_i\ne v_{ic}$, lies outside $D$ and is caught only if uncompensated; when disagreements share a magnitude, the cells go to the image (Section~\ref{sec:adjud}).

\subsection{What a Residual Says}
\label{sec:signature}
\begin{proposition}[Signature of omitted rows]
\label{prop:signature}
If a set $\Omega$ of constituent lines is omitted, none duplicated, and all other terms of Eq.~\eqref{eq:decomp} vanish in every column, the vector of residuals across the county's columns is $\boldsymbol d=-\sum_{\ell\in\Omega}\boldsymbol v_\ell$. A single omitted line is detected if and only if it prints a nonzero figure in some column, and then its figures are determined; which printed line carries them is not, and must be read from the page. From $\boldsymbol d$ alone, a single omitted line valued $-\boldsymbol d$, any set $\Omega$ with $\sum_{\ell\in\Omega}\boldsymbol v_\ell=-\boldsymbol d$, and reading errors with $\sum_ie_{ic}-e_{tc}=d_c$ in every column are observationally equivalent.
\end{proposition}

Within a column, decimal notation shapes the residual. Write each figure by its digits, place $p$ carrying $10^p$.
\begin{lemma}[Place-value syndromes]
\label{lem:syndrome}
Let a correctly scoped column have exactly one wrong number, a cell or the total, misread or misprinted. (i)~If a digit $a$ at place $p$ is read or printed as $a'$, then $d_c=\pm(a'-a)10^p$ (plus for a cell, minus for the total), so $|d_c|$ has exactly one nonzero digit, which gives $|a'-a|$ and $p$. (ii)~If the number's digits appear in another order, $9\mid d_c$; adjacent digits $b,a$ at places $p+1,p$ appearing as $a,b$ give $d_c=\pm9(a-b)10^p$. (iii)~Two digit substitutions alone in a column cancel if and only if they fall at the same place and move $d_c$ by opposite amounts.
\end{lemma}

A residual also bears odds on which printed number is wrong.
\begin{proposition}[One among $R+1$]
\label{prop:which}
Suppose each of the $R+1$ printed numbers of a column, the $R$ cells and the total, is defective independently, with probability $\beta_b$ for a cell and $\beta_t$ for the total. Conditional on the column containing exactly one defective number, the probability that it is the total is
\begin{equation}
 \frac{\beta_t(1-\beta_b)}{\beta_t(1-\beta_b)+R\beta_b(1-\beta_t)}\approx\frac{\beta_t}{\beta_t+R\beta_b},
 \label{eq:which}
\end{equation}
which is $1/(R+1)$ when $\beta_t=\beta_b$. Over failed columns with lengths $R_c$ and one defect each at equal rates, the expected number whose total is the defective figure is $\sum_c1/(R_c+1)$.
\end{proposition}

The prior odds are $R$ to one against the total, which is why a residual is read as a body problem until the total has been reinspected; Section~\ref{sec:statewide} tests the rule against the canvass.

\subsection{Detection, False Acceptance, and Silent Error}
\label{sec:fa}
Call a cell \emph{defective} if it is unreadable, misread, or misprinted relative to the canvass (the official county returns of Section~\ref{sec:statewide}), and suppose defects occur independently with probability $\theta$ per cell. A numeric defect changes the cell by a nonzero amount drawn, independently of everything else, from a law $q$ on $\mathbb{Z}\setminus\{0\}$; let $\theta_e\le\theta$ be the rate of numeric defects and
\begin{equation}
 \kappa=\max_{k\ge2}\Pr\Big[\textstyle\sum_{i=1}^{k}\varepsilon_i=0\Big],
 \label{eq:kappa}
\end{equation}
the cancellation probability, where the $\varepsilon_i$ are independent with law $q$. Assume the total is read correctly and is the canvass figure, and that $\mathrm{Cred}_c$ holds.

\begin{lemma}[Cancellation]
\label{lem:kappa}
For every law $q$, $\kappa\le\tfrac12$. If $q$ is symmetric, $\kappa=\Pr[\varepsilon_1+\varepsilon_2=0]=\sum_mq(m)^2$, and the maximum in Eq.~\eqref{eq:kappa} is attained at $k=2$.
\end{lemma}

\begin{corollary}[Place and digit]
\label{cor:place}
If every numeric defect is one substitution, at place $p$ with probability $w_p$ and changing the cell by $\delta10^p$ with $\delta$ drawn from a symmetric law $P$ independent of $p$, then $\kappa=\big(\sum_pw_p^2\big)\big(\sum_\delta P(\delta)^2\big)$: the collision entropy $-\log\kappa$ of the error law is the sum of those of place and digit. A uniform digit misread as a uniformly chosen other digit has $\sum_\delta P(\delta)^2=19/270$.
\end{corollary}

\begin{proposition}[Detection and false acceptance]
\label{prop:detect}
For a column of $R$ cells,
\begin{equation}
 1-(1-\theta)^R-\tbinom R2\theta_e^2\kappa\;\le\;\Pr[\text{fail}]\;\le\;1-(1-\theta)^R,
 \label{eq:detect}
\end{equation}
and the probability that a column passes while containing a defect is at most $\binom R2\theta_e^2\kappa\le(R\theta_e)^2\kappa/2$, or $\binom R2\theta_e^2\kappa/(1-\theta)^R$ conditional on passing.
\end{proposition}

A user of the data cares about cells, not columns. Let $Z$ be the number of wrong cells in a column and call $s=\mathbb E[Z\,\mathbf 1\{\text{pass}\}]/(R\Pr[\text{pass}])$ the \emph{silent error rate} of the cells of \Exact\ columns; without the gate it is $\theta_e$.

\begin{theorem}[Silent error]
\label{thm:silent}
Let $p_2=\sum_mq(m)q(-m)$ and $\rho=\theta_e/(1-\theta)$. For every $R$,
\begin{equation}
 (R-1)\,p_2\,\theta_e^{2}(1-\theta)^{R-2}\;\le\;s\;\le\;(R-1)\,\kappa\,\rho^{2}(1+\rho)^{R-2}.
 \label{eq:silent}
\end{equation}
If $q$ is symmetric, $p_2=\kappa$, the bounds differ by a factor of at most $(1-\theta)^{-2(R-1)}$, and $s=(R-1)\kappa\theta_e^2\,(1+O(R\theta))$ for $R\theta\le1$.
\end{theorem}

The gate turns a first-order error rate into one of at most second order, exactly second order with constant $(R-1)\kappa$ when errors are symmetric, but the constant grows with $R$: the sum protects a 45-row column far better than a 700-row one. Two channels lie outside the model: a misread total matched exactly by one cell defect, of probability about $R\theta_e\pi_t\kappa$ for a total-misread rate $\pi_t$ the test bed does not measure, and the sum-preserving patterns of Proposition~\ref{prop:kernel}. The gate also refuses to repair a column from its total, and the refusal is arithmetic, not caution.

\begin{proposition}[Repair by arithmetic]
\label{prop:repair}
Let one cell of a column be unreadable, the other $R-1$ carry numeric defects independently with probability $\theta_e$, and the total be misread with probability $\pi_t$. Writing into the unreadable cell the value that makes the column foot satisfies (G2) whatever the errors, and the written value is wrong exactly when the other $R$ numbers carry a nonzero net error; the expected share of wrong cells lies between $(R-1)\theta_e\big(1+(1-\theta_e)^{R-2}(1-\pi_t)\big)/R$ and $\big(2(R-1)\theta_e+\pi_t\big)/R$.
\end{proposition}

A repaired county-sized column carries nearly twice the error rate of unchecked readings behind an empty certificate; a gated one, the rate of Theorem~\ref{thm:silent}.

\subsection{Where the Blocks Are}
\label{sec:infer}
The footing count of Eq.~\eqref{eq:foot} has an exact statistical meaning under a plain model $\mathcal M$ of a printed table: given the cut, each constituent-line reading in column $c$ is drawn independently from a law $f_c$ on $\mathbb{Z}_{\ge0}$, positive at every reading, and in each column the total line of each block reproduces the block's column sum with probability $1-\varepsilon_c\in(0,1)$ and is otherwise drawn from $f_c$, like a constituent figure, all independently.

\begin{theorem}[Segmentation as inference]
\label{thm:infer}
Under $\mathcal M$, for every admissible $K$-cut $\mathbf j$ of fully read lines, with $B_k$ its $k$th block,
\begin{equation}
 \log\Pr[x\mid\mathbf j]=\mathrm{const}(x)+\sum_{k=1}^{K}\sum_{c=1}^{C}\omega_c(x_{j_kc})\,\mathbf 1\{B_k\text{ foots in }c\},
 \label{eq:loglik}
\end{equation}
where $\omega_c(u)=\log\big(1+(1-\varepsilon_c)/(\varepsilon_cf_c(u))\big)$. (i)~Eq.~\eqref{eq:dp} with block score $\sum_c\omega_c(x_{jc})[\text{foot}_c]$ returns a maximum-likelihood cut; if the $\varepsilon_c$ are equal and the $f_c$ uniform on ranges of one length holding every reading, that score is a constant multiple of $\phi$. (ii)~Under a uniform prior on admissible cuts, Eq.~\eqref{eq:dp} with $(\max,+)$ replaced by $(+,\times)$, run forward and backward, gives the posterior probability that line $j$ closes block $k$, for all $j$ and $k$, in the time of Lemma~\ref{lem:dp}. (iii)~A block whose last line is not a total line foots in column $c$ with probability at most $\max_uf_c(u)$, independently across columns.
\end{theorem}

Part~(i) makes the footing count, the deployed objective without its small cue terms, which act as a log prior, exact maximum likelihood when footing columns carry equal evidence, and weighs each: footing on a total $u$ multiplies the likelihood by $1+(1-\varepsilon_c)/(\varepsilon_cf_c(u))$, large for an improbable figure and small for a common one, so a block that foots only on zeros and small figures is weakly anchored. By~(ii), a cut that foots in $g$ fewer columns at equal weights $\omega$ is $\mathrm e^{\omega g}$ times less probable; (iii)~is the hypothesis of the recovery theorem (Theorem~4 of the supplement).

\subsection{Withholding and the Verification Design}
\label{sec:design}
\begin{proposition}[Size bias]
\label{prop:sizebias}
Let columns $c=1,\dots,C$ have lengths $R_c$ and fail independently with probability $p(R_c)$. The expected share of cells in failed columns is the size-biased mean $\Lambda=\sum_cR_c\,p(R_c)/\sum_cR_c$; whenever $p$ is nondecreasing, $\Lambda\ge\frac1C\sum_cp(R_c)$, the expected share of columns that fail. For $p(R)=1-(1-\theta)^R$, $\Lambda\le\theta\bar R_{\mathrm{sb}}$, with $\bar R_{\mathrm{sb}}=\sum_cR_c^2/\sum_cR_c\ge\bar R=\frac1C\sum_cR_c$, and $\Lambda\approx\theta\bar R_{\mathrm{sb}}$ when $\theta R_{\max}\ll1$.
\end{proposition}

Verifying against a printed total is a group test in Dorfman's sense~\cite{dorfman}: a block of cells is pooled, the pool is positive when its sum fails, and only positive pools are examined further. Books print sums at several levels, municipality, county, and state, and a top-down protocol uses them all: corroborate each top block's total, corroborate the totals inside a block only when it fails, and review or withhold the cells of each failing bottom block. Count one unit of work per total corroborated and $\gamma$ per cell reviewed or withheld, and let a block fail exactly when it holds a defect, as in Proposition~\ref{prop:detect} up to cancellations, which lower the work but are silent errors.

\begin{theorem}[Nested printed sums]
\label{thm:nested}
With $h$ levels of blocks of $M_1\mid M_2\mid\dots\mid M_h$ cells and $p(M)=1-(1-\theta)^M$, the expected work per cell is
\begin{equation}
 E=\frac1{M_h}+\sum_{k=1}^{h-1}\frac{p(M_{k+1})}{M_k}+\gamma\,p(M_1),
 \label{eq:nested}
\end{equation}
and $(1-\theta)^{M_h}B\le E\le B$, where $B=1/M_h+\theta\sum_{k<h}M_{k+1}/M_k+\gamma\theta M_1\ge B_h^\star=(h+1)(\gamma\theta^h)^{1/(h+1)}$, with equality at the real design $M_k=(\gamma/\theta)^{k/(h+1)}/\gamma$; the least $E$ over integer designs is $B_h^\star(1+o(1))$ as $\theta\to0$. At $\gamma=1$, blocks of $3,9,27,\dots$ cells over $\lceil\log_3\theta^{-1}\rceil$ levels take $E\le\theta(4+\frac{3}{\ln3}\ln\theta^{-1})$, and any protocol that finds the defective cells from pass--fail outcomes takes on average at least $H_2(\theta)$ steps per cell, $H_2$ the binary entropy in bits.
\end{theorem}

For $h=1$ the theorem is Dorfman's: blocks of $m^\star\approx(\gamma\theta)^{-1/2}$ rows at about $2\sqrt{\gamma\theta}$ per cell, and $E\to\gamma$ as the block grows, so a pool much longer than $m^\star$ has nearly every cell reviewed or withheld. Each further printed level raises the exponent of $\theta$, from $\frac12$ to $\frac h{h+1}$; with about $\log_3\theta^{-1}$ levels the work comes within a factor $3\ln2/\ln3\approx1.89$ of the pass--fail lower bound as $\theta\to0$, while with a fixed number of levels the ratio grows without bound. Section~\ref{sec:model} evaluates the two and three levels a book can print.
\section{Evaluation}
\label{sec:eval}
\subsection{Test Bed, External Evidence, and Reproducibility}
\label{sec:repro}
We evaluate on five Wisconsin Blue Book volumes: 1925, which reports the 1924 gubernatorial primary, and 1954, 1956, 1958, and 1962, which report the 1952, 1954, 1956, and 1960 general elections, eight contests in all. Each volume covers 71 counties, and the contests have 31 candidate or printed-category columns between them, so the test bed has $31\times71=2201$ county-by-candidate verification columns, from Florence County's eight rows to Milwaukee County's 375 to 693. The frozen inputs are the project's final per-volume data, withheld-column file, and repair log; none is an independent gold standard, so the evaluation adds four kinds of evidence the workflow did not produce:

\emph{Official county returns}: county-level canvass figures for the two major candidates in seven of the eight contests, taken from the ``Results by county'' tables of the corresponding Wikipedia articles~\cite{countyreturns}, each verified to sum exactly to the published statewide total before use; for a flagged or withheld column the read printed total and the body sum are compared with the official figure separately, which shows which of the two the book printed correctly. \emph{Typed reference pages}: two student typists independently keyed six pages without access to any machine output: Wisconsin Blue Book 1940 pp.~500 and 602, Minnesota Legislative Manual 1889--90 pp.~340 and 377, and two low-resolution 1887 Iowa renderings that proved not adjudicable; disagreements on the four legible pages were adjudicated by printed totals and the page image. For the same pages we took the project's earlier single-pass reader, Amazon Textract's raw table extraction with no pre- or post-processing~\cite{textract}, since the vision workflow has no output for these pages. \emph{Hard pages}: three continuation pages of the 1888 Iowa returns, read for this benchmark at magnification and checked against their printed totals, on which Textract's stored output, the open-source readers, and a stored pass of the workflow's own reader are scored (Section~\ref{sec:human}). \emph{Source-page cases}: three Delaware County columns of the 1894 Iowa register, all 63 constituent values and the three totals rechecked against printed pp.~126--127, and one flagged Langlade County column, all 24 values and the total against printed p.~455.

The \Exact\ share is a classification rate, not a measured cell accuracy; TEDS, GriTS, and latency were not measured. The replication materials, including the reader benchmark and a numerical check of every result of Section~\ref{sec:theory}, reproduce every count, comparison, and fit below from a data snapshot of September 27, 2026, but neither the recognition run nor the exact-sum classifications themselves (Section~\ref{sec:disc}). One inconsistency in the inputs: the repair log describes the 118 flagged columns as staying withheld, whereas the final data release them with caution flags; we use final-data membership to measure release.
\subsection{Gate Outcomes on the Test Bed}
\label{sec:results}
Table~\ref{tab:model} gives the outcome by volume. The release holds 2160 columns and 91\,216 long-format records: of the 2201 columns the gate classifies 2042 as exactly reconciling (92.78\%), releases 118 with a source-discrepancy flag (5.36\%), and withholds 41 (1.86\%). The 1925 volume releases nearly everything yet has 66 of its 497 columns flagged, 44 of them in the primary's three Republican columns, where the book's own arithmetic is worst; 11 of the 41 withheld columns are Milwaukee's. Column counts hide how much withholding removes, as Proposition~\ref{prop:sizebias} predicts: Proxmire is released in 68 of 71 counties in 1954 (95.77\%) yet the released vote is 362\,151 of a statewide 560\,747 (64.58\%), because the gaps are Milwaukee-sized.

\begin{figure}[t]
\centering
\includegraphics[width=0.9\columnwidth]{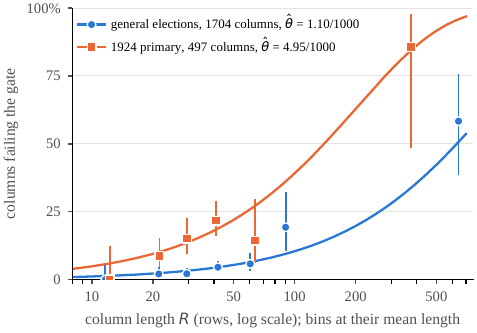}
\caption{Gate failure rises with column length as the model predicts: share of columns failing the gate, by length, for the four general-election volumes (1704 columns) and the 1924 primary (497), with 95\% Wilson intervals and the curves $1-(1-\hat\theta)^R$ of Table~\ref{tab:model}; the longest bins are Milwaukee County's.}
\label{fig:accuracy}
\end{figure}
\subsection{The Defect Model Against the Test Bed}
\label{sec:model}
Two predictions of Section~\ref{sec:theory} need nothing but the release itself to test: Proposition~\ref{prop:detect} says the probability of failing the gate rises with column length as $1-(1-\theta)^R$, and Proposition~\ref{prop:sizebias} says the share of cells in failed columns exceeds the share of columns. Every column has a known length and outcome; treating outcomes as independent, we fit the maximum-likelihood rate $\hat\theta=\arg\max_\theta\big[\sum_{c\,\mathrm{fails}}\log(1-(1-\theta)^{R_c})+\sum_{c\,\mathrm{passes}}R_c\log(1-\theta)\big]$ separately to the four general-election volumes and to the 1924 primary volume, whose typography and candidate set differ; the log-likelihood is strictly concave in $-\log(1-\theta)$, so each $\hat\theta$ is unique (Lemma~4 of the supplement). Two caveats apply. First, the likelihood uses the upper bound of Eq.~\eqref{eq:detect}, so $\hat\theta$ is a detected-defect rate, which understates the defect rate by up to 7\% for the general-election volumes and 22\% for the primary if cancellations are at the worst case Proposition~\ref{prop:detect} allows ($\kappa=\frac12$), and by 2\% or less at the digit-shaped $\hat\kappa$ of Section~\ref{sec:statewide}; bounds evaluated at $\hat\theta$ are optimistic in the same proportion. Second, a lost row fails all of a county's columns at once, so beside each profile-likelihood interval we report a county-clustered bootstrap interval (2000 replicates); the tests below are not cluster-adjusted.

\begin{table}[t]
\caption{The Defect Model by Volume}
\label{tab:model}
\centering\scriptsize
\setlength{\tabcolsep}{2.1pt}
\begin{tabular}{@{}lrrrrrlr@{}}
\toprule
& & \multicolumn{2}{c}{Failed} & \multicolumn{2}{c}{Rows} & & Cells failed (\%)\\
\cmidrule(lr){3-4}\cmidrule(lr){5-6}
Volume & Cols. & Flag & With. & $\bar R$ & $\bar R_{\mathrm{sb}}$ & $\hat\theta$ (95\%) & obs./model\\
\midrule
1925 & 497 & 66 & 13 & 37.8 & 85.1 & 4.95 (3.9--6.1) & 26.3 / 26.4\\
1954 & 639 & 20 & 2 & 45.5 & 149.0 & 0.80 (0.5--1.2) & 9.4 / 10.0\\
1956 & 213 & 0 & 8 & 46.3 & 161.8 & 0.86 (0.4--1.6) & 9.9 / 11.3\\
1958 & 568 & 20 & 14 & 47.2 & 167.6 & 1.47 (1.0--2.1) & 24.4 / 17.5\\
1962 & 284 & 12 & 4 & 49.0 & 176.4 & 1.26 (0.7--2.0) & 15.5 / 16.2\\
\midrule
General & 1704 & 52 & 28 & 46.7 & 161.6 & 1.10 (0.9--1.4) & 15.6 / 13.7\\
All & 2201 & 118 & 41 & 44.7 & 147.0 & 1.79 (1.5--2.1) & 17.6 / 18.2\\
\bottomrule
\end{tabular}
\par\smallskip
\parbox{\columnwidth}{\footnotesize Columns; flagged and withheld columns; mean and size-biased mean column length; fitted rate $\hat\theta$ per 1000 cells with its profile-likelihood 95\% interval; and observed and predicted shares of cells in failed columns.}
\end{table}
Table~\ref{tab:model} and Fig.~\ref{fig:accuracy} give the fit. The general-election volumes fail at a rate consistent with $\hat\theta=1.10$ defects per 1000 cells (profile-likelihood 95\% interval $0.87$--$1.36$; cluster bootstrap $0.81$--$1.40$), and the failure share climbs from none of the 72 columns with at most 15 rows to 58\% of the 24 Milwaukee columns. The 1924 primary fails four and a half times as often, $\hat\theta=4.95$ ($3.9$--$6.1$; cluster bootstrap $3.6$--$6.4$); its failures belong to the book, 66 of its 79 being adjudicated source discrepancies. Freeing the exponent in $p(R)=1-(1-\theta)^{R^{\nu}}$ gives $\hat\nu=1.15$ ($0.94$--$1.33$) for the general-election volumes and $1.02$ ($0.67$--$1.34$) for the primary, where the model says $1$; a length-independent rate ($\nu=0$) is rejected outright (likelihood-ratio statistics $78.9$ and $28.2$ on one degree of freedom). Binned Pearson statistics of $9.2$ over seven length bins and $6.5$ over six show no lack of fit (parametric-bootstrap $p$-values $0.15$ and $0.25$, the shortest bins expecting fewer than one failure). The size bias of Proposition~\ref{prop:sizebias}, computed exactly since Milwaukee's columns are too long for the linear form, predicts that a rate failing 7.2\% of columns puts 18.2\% of cells into failed columns; the observed share is 17.6\%, and the prediction holds within a few points in every volume except 1958, where Milwaukee's four major-candidate columns all failed.

At the point estimate the theorems quantify the design. With $\theta_e=\theta=\hat\theta$, Theorem~\ref{thm:silent} puts the silent error rate of the cells of an \Exact\ column of the general-election mean length, 46.7 rows, at most about $1/38$ of the ungated rate for every error law ($\kappa\le\frac12$, Lemma~\ref{lem:kappa}), between $1/42$ and $1/38$ for a symmetric law with $\kappa=\frac12$, and between $1/478$ and $1/432$ for the digit-shaped $\hat\kappa=0.044$ of Section~\ref{sec:statewide}. For a 693-row Milwaukee column the general bound guarantees only a reduction to 0.81, and weighted by length over the 934 accepted columns the canvass tested, the upper bounds are $1/12$ and $1/134$. Repaired from its total, a column would carry nearly twice the ungated rate instead (Proposition~\ref{prop:repair}): written as 267, the destroyed Polk County cell of Section~\ref{sec:gate} would have absorbed any misreading among its 75 siblings, about an 8\% chance at the Wisconsin general-election rate. At $\gamma=1$ the best single block is $m^\star=31$ rows at $0.066$ per cell, against $0.072$ for a 47-row county column and $0.53$ for Milwaukee's 693 rows (Theorem~\ref{thm:nested}); subtotals every 10 rows in blocks of 100 take $0.031$, and three levels, blocks of 5, 30, and 180 rows, $0.023$, within a factor of 1.9 of the $H_2(\hat\theta)=0.012$ below which no pass--fail protocol can go. Verifying blocks of 10 to 25 rows against printed municipal subtotals would put 1.1\% to 2.7\% of cells into failed blocks, against the 13.7\% the model predicts for whole county columns, and where a book prints such subtotals the better design is already on the page.

\subsection{Gate Decisions Against Official Returns}
\label{sec:statewide}
If the 71 printed county totals for a candidate sum to the official statewide total, the released sum must equal that total minus the printed totals of the withheld counties plus the residuals of the flagged columns; the identity can be tested for nine of 14 candidate--contest pairs, and it holds exactly in all nine.
The county-level test is sharper (Fig.~S1 of the supplement). The sum of released rows in an \Exact\ county column matches the official county figure in 933 of the 934 columns that can be tested; the exception is Rock County in the 1960 gubernatorial election, where the released Nelson column sums to 20\,690 against an official 20\,890 after reconciling to a read printed total, either a compensating pair of errors (Proposition~\ref{prop:kernel}) or a difference between the Blue Book and the canvass. The same comparison adjudicates the failed columns: of the 16 withheld columns with a recorded printed total and an official figure, the printed total equals the official figure in 15, and in one, Proxmire in Dane County in 1956, the body sum of 44\,851 is the official figure and the printed total of 44\,857 is the book's error. Of the 118 flagged columns, 35 have an official figure, and in all 35 the printed total is the official figure while the body sum is off by the logged residual. The release records no footing counts, so the segmentation is validated only through its consequences: a block cut at the wrong line would not sum to the official county figure. These are column-sum checks, blind to compensating errors within a column and covering neither the 1924 primary nor the minor candidates.

Two propositions predict those exceptions. The 51 failed columns the canvass can adjudicate have $\sum_c1/(R_c+1)=1.07$ ($0.77$ over the 35 flagged columns alone, the ones certain to contain a printed defect), so under equal per-number defect rates about one should have a defective total rather than a defective body (Proposition~\ref{prop:which}); exactly one does, Dane County in 1956, the modal outcome under the model ($\Pr[X=1]=0.37$), though the test has little power against unequal rates. Proposition~\ref{prop:detect} bounds the expected number of false acceptances among the 934 tested columns by $\sum_c\binom{R_c}{2}\theta_e^2\kappa/(1-\theta)^{R_c}$: at $\theta_e=\hat\theta$ and $\kappa=\frac12$ that is $1.8$ ($3.1$ at the interval's upper end), so one false acceptance is what the model expects. If reading errors are digit-shaped like the 118 flagged residuals, whose signs are balanced (61 positive, 57 negative) and whose symmetrized magnitude law $\hat p$ gives $\hat\kappa=\frac12\sum_{m>0}\hat p(m)^2=0.044$ ($0.041$ unsymmetrized; bootstrap $0.039$--$0.059$; Lemma~\ref{lem:kappa}; Corollary~\ref{cor:place} with a uniformly wrong digit at the places those residuals occupy gives $0.037$), the expectation falls to $0.16$ and one occurrence is a 15\% event. That leans toward a coordinated cause for Rock County: a body misprint carried into the book's own total, which no sum check can see.

\subsection{Observed Failures and Targeted Review}
\label{sec:review}
Table~S3 of the supplement collects the failure modes we met, the signal that exposed each, and what became of the data: structural failures (rows 1--5) are found by geometry and the cross-column signature, not by any single sum; reading failures (rows 6--7) by the sum, a second witness, or magnification; source and blind-spot cases (rows 8--9) no reader can fix. The repair log holds 174 investigation entries covering 167 columns of the test bed: 118 adjudicated source discrepancies, each recording two readers in agreement; 15 applied cell corrections and one correction to a read printed total; 29 reader disagreements, 10 unresolved cases, and one proposed correction rejected because the column still failed to reconcile. The flagged columns' residuals carry the signature of Lemma~\ref{lem:syndrome}: of the 45 with two or more digits, 36 have a single nonzero digit, where a spread uniform within each decade would give 4, and 3 of the other 9 are multiples of 9, as transpositions would be (1 expected by chance); with the 73 below 10, 112 of the 118 are consistent with one digit slip or one transposition. Langlade County's Quick column in the 1924 primary, whose 24 rows sum to 20 against a printed 21 on p.~455, can be checked on the page, and Proposition~\ref{prop:which} puts the odds at 24 to one that the body, not the total, is wrong.

\subsection{Readers Against the Same Pages}
\label{sec:human}
Table~\ref{tab:bench} gives the open-source readers the same input, a hand-cropped table panel with its header at native and twice native resolution: the two Minnesota pages the typists keyed, p.~455 of the 1925 Blue Book, and four panels of the 1954 Blue Book's pp.~652--653; Textract's earlier output and the typists are scored on the four keyed pages (Table~S2). The references are the typists' adjudicated transcription and the project's Wisconsin readings, kept only where each county column reconciles with its printed total or every cell was checked against the page. Each reader is shown at the configuration that files the most cells, matched as favorably as its output allows; the supplement gives the protocol and every panel.

The readers lose the table in different ways. A text recognizer returns lines without columns, and a blank cell, printed as leader dots, leaves no token, so the figures after it slide one column left: on p.~455, where 144 of the 448 cells are blank, docTR reads 98.6\% of the digits and files 50.3\% of the figures in place, though 96.7\% stand in the correct order within their row. A table extractor keeps blanks but draws its own grid, and the grid is where it fails: SLANet-plus returned Ramsey County's 48 rows as a single row of six cells holding 143 of the 276 figures, yet the PaddleOCR words it was given, filed by the comb, put 96.7\% of that page's figures in place. The losses gather where tables are wide and sparse; on the 1954 panels, two columns wide with zeros printed, the strongest readers file nearly everything.

Structure is recoverable from the same words. A word list has no pixels for the comb of Section~\ref{sec:geom}, so the benchmark centers the $C$ windows ($C$ given by the header) where the most figures align at their right edges and applies the arithmetic of the gate, (G1) and (G2), an empty window adding nothing to a sum. Filed this way, docTR's words put 97.9\% of the figures and 24 of the 30 columns in place, all 13 added columns on the two pages with blank cells; on the five Wisconsin panels they conflict with the project's reading in none of the 761 figures they place, an independent check of that reference. On the comb-filed words of the five distinct open-source readers (SLANet-plus's are PaddleOCR's), 100 columns close on their panel and 27 are correct; the check releases 25 of the 27 and 1 of the other 73. Run on the lines as the readers return them, the same check releases 8 columns, 2 of them wrong (Table~S2).

A commercial table reader is stronger on clean pages: on the two Minnesota pages both were given, Textract alone files 445 of the 446 figures and 5 of the 6 county columns, the comb on docTR's words 428 and 3. For a reader that strong, what the check catches is the error that looks correct. Textract and typist~B misread one and two of the 1356 keyed figures, both reading the broken 42 at Waterford as 12, and the check withholds that column from both; on the keyed grids it releases 46 of the 47 closing columns that are correct and 1 of the 10 that are wrong. Typist~A misfiles 72 figures, 69 of them belonging to another row of the same column (57 to the adjacent row), and puts six more on heading rows; six of these lie in a column on 1940 p.~602 that still sums to its printed total, and the check releases it, a sum-preserving permutation that only subtotals could catch (Proposition~\ref{prop:kernel}, Corollary~\ref{cor:subtotal}). The one inexact release from comb-filed words, a Langlade column in img2table's, offsets five dropped figures worth 50 with a 30 read as 80, a cancellation Proposition~\ref{prop:detect} makes likelier the more errors a reader makes. No page reader attempts structure above the page: county blocks that cross page breaks and rows that never reached the table (Sections~\ref{sec:segment} and~\ref{sec:diag}); the workflow's own first pass read 5612 of the 5630 audited cells of the 1940 Milwaukee table.

The hard pages test the opposite case (Table~\ref{tab:bench}, bottom; protocol, reference, every reader, and the disputed cells in the supplement, Tables~S4--S5). They are the three continuation pages of the 1889 Iowa Official Register (the 1888 election) for which the workflow's reader has stored output, pp.~113, 145, and 171: one-bit scans at 118 pixels per inch, 16 candidate columns under rotated names, no row labels, two county blocks a page. Their reference was read for this benchmark at four to fourteen times magnification and certified by the printed totals: of the 61 county columns whose rows and total all stand on the page, 59 reproduce the total exactly and two are the book's discrepancies in Madison County, and the continued Floyd and Scott blocks reproduce their county totals together with their rows on pp.~115 and 169. Textract's stored output reads all three tables sideways, each row of its grid a column of the page read from the bottom up; 76.7\% of the digits occur somewhere in it, but 1 of the 1494 figures stands in its cell and no county column is correct. The four open-source text recognizers, each given four scales, reproduce 30--71\% of the digits at their reported configurations and file 4--19\% of the figures from their own lines and 13--66\% through the comb (Table~S4), with no county column correct. The workflow's reader, in the stored pass the project's pipeline kept (Claude Sonnet~4.6 returning each table as JSON with its total line), files 99.2\% and gets 67 of the 77 county columns exact. Each of the 8 closing columns that holds one of its wrong cells fails the check, which releases the other 51, all exact, and withholds Madison's two discrepancies and Scott's 16 columns, whose total covers rows on the previous page. From the open-source readers it releases 2 columns at these configurations and 23 over every configuration run, none exact, 16 of them on a figure that is not the column's printed total, so the premise of a correctly read total fails. Where it holds, the conditional bound of Proposition~\ref{prop:detect} is vacuous, its denominator $(1-\theta)^R$ near zero when nearly every column holds a defect.

\subsection{Deployment Beyond the Test Bed}
\label{sec:deploy}
The deployment behind the test bed added 491\,359 gated precinct rows (Iowa 288\,777; Wisconsin 202\,582) from eleven Iowa Official Registers (the elections of 1889--1898 and 1916) and the Wisconsin Blue Books to a 9.3-million-row precinct dataset, raising Iowa's coverage from 51 to 62 of 65 elections and Wisconsin's from 81 to 91 of 93, with no digit bent and every failed column withheld whole. Its audits, though the workflow's own, recompute from the shipped files that all 932 \Exact\ columns of the 1892 Iowa volume and all 2222 of 1898 reconcile, and that the 71 county totals of the 1940 Wisconsin volume as read sum to the published statewide figures, to the vote, for all five leading candidates.
\section{Discussion}
\label{sec:disc}
The certificate is narrow: an \Exact\ column's sum equals a printed total that agrees with the official canvass in 933 of the 934 cases we could test, a certificate about sums and not about cells (Proposition~\ref{prop:kernel}); the one exception is within what Proposition~\ref{prop:detect} allows. In use, an \Exact\ column's sum may be taken as the county return; a \Source\ column's printed total may be taken as the county return (35 of 35), its rows being correct only in aggregate to within the residual; a \Hold\ column contributes nothing, so a sum over the release is a lower bound on any statewide total, not an estimate of it. Of the nine failure modes in Table~S3, two are reading failures; the rest were found by mechanisms that operate on the table, not on a cell, as Table~\ref{tab:bench} would lead one to expect.

Two design lessons follow. First, the verification block should follow the defect rate among the sums a book prints, not the county boundary. One unreadable cell in 693 rows withholds Milwaukee County, and 52 unread cells withhold 4882 records, because the number of cells lost scales with the size-biased mean column length (Proposition~\ref{prop:sizebias}) while the sum's protection weakens in the same proportion (Theorem~\ref{thm:silent}). Theorem~\ref{thm:nested} gives the optimum, 31 rows for a single level and subtotals every 10 rows in blocks of 100 for two, and Corollary~\ref{cor:subtotal} adds that such blocks detect the displacements a county sum cannot. Second, a better recognizer and the structure around it are complements, not substitutes: the bound of Theorem~\ref{thm:silent} falls with $\theta^2$, so halving $\theta$ roughly quarters the silently wrong cells among released ones, while a missing row, a wrong total, or a miscut county does the same harm as before.

\emph{Scope and impact.} Little in Sections~\ref{sec:problem}--\ref{sec:theory} is specific to elections: any printed register whose blocks close on a total has the same residual decomposition, kernel, segmentation likelihood, and verification design, and the election-specific inputs are $K$, the typing vocabulary $\mathcal T$, and the straight-ticket prior behind the coherence screen. As noted in Section~\ref{sec:intro}, TableVision has applications beyond electoral data. To give one specific example, a century of American administrative statistics survives mainly as such tables, and on the evidence of Table~\ref{tab:bench} the obstacle is structural and silent error, not illegible digits. One such source is the Municipal Year Book~\cite{myb}: for example, its 1950 volume prints 21 major tables for 2042 cities over 5000 people, from officials' salaries to police and fire statistics and revenues, expenditures, and debt. A table that prints components and their total carries the check equation unchanged; one that prints a single figure per city gets the geometry, typing, constraints, and provenance trace, and its columns enter the \Unverified\ state under the two-reader attestation $\mathrm{Agr}_c$. Either way the product is a dataset that states, column by column, what was checked and how.

\emph{Limitations.} The test bed consists of deployment tasks rather than a random sample of historical tables, and the outcomes need not generalize to handwriting, curved pages, or documents without redundant arithmetic; whole-column withholding selects the released data by legibility and table complexity, so excluded counties and candidates are not missing at random. The theorems hold under stated idealizations: the defect model treats cells as independent, which the clustering of the 1925 volume's flagged columns contradicts, so its rate is a summary rather than a mechanism, and the segmentation model treats readings, wrong totals among them, as independent draws from one law per column. Three quantities the theory names were not measured: the total-misread rate $\pi_t$ (one read total corrected among 2201 columns is a floor, not an estimate), the footing counts and evidence weights of the segmentation theorems (Theorem~\ref{thm:infer}; Theorem~4 of the supplement), and the reading-error law behind $\kappa$, for which the adjudicated misprints stand in; none can be recovered, because the audited release holds no row-level inventory of first-pass readings and read totals. Model versions, prompts, raw independent reads, and full image-coordinate traces are incomplete or absent from the audited release, so the scripts reproduce every count and fit but not the recognition run; and the 1925 file retains two repeated place labels across seven candidates (14 excess rows), so we count records and verification columns, not distinct precincts. The workflow's reader is scored only on the three hard pages, from the stored pass the project's pipeline kept, and their reference was read for this benchmark by an AI model of the same family and checked against the printed totals rather than keyed by independent typists; a misreading that leaves its column's sum intact would pass that check as it passes the gate.

\section{Conclusion}
Historical tables are imperfect as a rule, and recovering them is not mainly an OCR problem. The best open-source readers reproduce nearly every digit of a clean page and lose the table; a commercial reader turns a hard page sideways, its columns into rows of fused figures; and even a reader as careful as a typist leaves the data untrustworthy until someone decides where each county begins and ends, whether every row is present under the correct name, whether the printed total is credible, and what to say about a column that does not close. TableVision makes those decisions with the table's own arithmetic, used as evidence and never as a target, and leaves every column exactly reconciling, flagged under a two-reader attestation, or withheld. The theory proves what the certificate covers and how verification should be laid out; on a 2201-column test bed the testable predictions hold, the decisions survive official returns the workflow never saw, and the deployment added nearly half a million gated rows to a 9.3-million-row dataset. The design lesson should travel, and the method with it: verify against printed sums over blocks near $\theta^{-1/2}$ rows where a book offers them, use every level of subtotal it prints, and treat a century of printed registers, from election returns to the Municipal Year Book's city statistics, as datasets to be certified rather than pages to be recognized.


\clearpage
\setcounter{equation}{0}\setcounter{figure}{0}\setcounter{table}{0}
\setcounter{theorem}{0}\setcounter{proposition}{0}\setcounter{corollary}{0}\setcounter{lemma}{0}
\renewcommand{\theequation}{S\arabic{equation}}
\renewcommand{\thefigure}{S\arabic{figure}}
\renewcommand{\thetable}{S\arabic{table}}
\begin{center}
{\Large\bfseries Supplementary Material}\\[4pt]
{\large for ``TableVision: Recovering Verified Structure from Historical Tables''}\\[3pt]
Todd~Y.~Zhou and James~M.~Snyder,~Jr.
\end{center}
\vspace{2pt}
\section*{Supplementary Figure}
\begin{figure}[!t]
\centering
\includegraphics[width=\columnwidth]{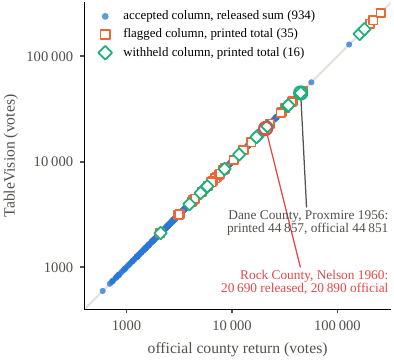}
\caption{The transcription against the official canvass, which the workflow never saw (Section~VI-D of the main text): the released sum of the 934 accepted columns and the read printed total of the 35 flagged and 16 withheld columns with an official county figure. Of the 985, 983 agree to the vote; one printed total is off by 6 (Dane County 1956, the book's error) and one accepted sum by 200 (Rock County 1960), either a compensating pair of errors (Proposition~1) or a discrepancy between the book and the canvass.}
\label{supp:fig:canvass}
\end{figure}

\section*{The Reader Benchmark of Table~I}
\emph{Panels and references.} Table~I of the main text gives the open-source readers seven table panels: pp.~340 and 377 of the Minnesota Legislative Manual for 1889--90 (Dakota and Ramsey Counties; 170 and 276 figures), the two Minnesota pages the typists keyed; p.~455 of the 1925 Wisconsin Blue Book (the 1924 gubernatorial primary in Lafayette, Langlade, and Lincoln Counties; 304 figures); and pp.~652--653 of the 1954 Wisconsin Blue Book (the primary vote for governor by precincts in Waukesha, Waupaca, Waushara, Winnebago, and Wood Counties, two panels to a page, cropped panel by panel; 126, 124, 110, and 104 figures). The Minnesota references are the typists' adjudicated transcription (Section~VI-F). The Wisconsin references are the project's readings, kept only for counties whose every column reconciles exactly with its printed total (Waukesha's across the break from p.~651) or, where a printed total does not close, whose every cell was checked against the page image. Three such columns keep their printed totals: Langlade's Bucknam and Quick on p.~455 (46 and 21, against cells that sum to 44 and 20) and Winnebago's Kohler on p.~653 (23\,353, against 23\,359); so does Dakota County's typed Harrison (Governor) column on p.~340 (261, against 240). An engine that never saw these readings agrees with them: docTR's words, filed by the comb, put 761 of the 768 Wisconsin figures in their cells, every one equal to the reference, and leave the other seven cells empty. Leader dots and dashes are blanks; the 1954 pages print their zeros, which are figures. The panels hold 1214 figures and 30 county-by-candidate columns with a printed total on the panel (Ramsey County continues overleaf). Amazon Textract (its earlier full-page output) and the typists are scored on the four keyed pages instead (Table~S2), the Minnesota pair and Wisconsin 1940 pp.~500 and 602 (1356 figures, 34 county columns), whose page images are not in the package; since the typed reference was adjudicated from the typists' own two transcriptions, their scores are upper bounds.

\emph{Readers and configurations.} Every open-source reader receives the same crop, header included, at native resolution and enlarged twice (Lanczos): Tesseract~5.3.4 with its LSTM engine in page-segmentation modes 3, 4, 6, and 11, each with the packaged and the tessdata\_best English models (16 configurations); EasyOCR~1.7.2; the PP-OCRv4 detection and recognition models of PaddleOCR through RapidOCR~1.4.4, with its per-character boxes; docTR~1.1.0 with the DB-ResNet50 detector and the CRNN-VGG16 recognizer; img2table~2.0.0 with Tesseract, implicit rows and columns, and borderless tables allowed; and PP-Structure's SLANet-plus table recognizer through RapidTable~3.0.2, given the PP-OCRv4 text boxes. The text layer of the 1925 scan is read with pdfplumber. For each reader the configuration with the most cells filed, pooled over its panels, is reported in Table~I and Tables~S1 and S2, a choice made on these panels that favors the reader; none was chosen per panel.

\emph{Scores.} A text reader's rows are its own lines (Tesseract modes 3, 4, and 6) or its words grouped by vertical position, and the cells of a row are its integers from left to right; a table reader's rows and cells are its own, empty cells included. A figure split at its thousands separator is rejoined. \emph{Characters} counts the reference digits reproduced under the one-to-one matching of reference figures to output figures, anywhere on the panel, that minimizes the total edit distance, a reading-order-independent count. \emph{Cells} counts the reference figures found in their own row and column as the reader files them, after reference rows are matched to output rows by the monotone alignment that shares the most figures and each panel is given the column offset that files the most figures. \emph{In order} (Table~S2) counts a figure if it is the $k$th figure of its row in both the reference and the output, which forgives every blank cell the reader lost. A row is exact if every figure is filed correctly and no figure stands in a blank cell; a county column is exact if every cell and the printed total are filed correctly and no figure stands in a blank cell.

\emph{The comb and the gate on words.} A word list has no rules and no pixels, so the comb of Section~IV-A cannot be measured from it. Because the figures are set flush right, the benchmark fits the comb to the right edges of the reader's figures instead: the $C$ windows are centered on the $C$ positions, at least one digit width apart, that the most figures share to within half a digit width (a dynamic program over the sorted right edges), $C$, the header's column count, being given and the digit width the median over the reader's figures. Rows are the reader's words grouped by vertical position; each figure is filed in the nearest window, and a number more than half a column pitch outside the comb belongs to the row label. On a line whose label reads ``Total,'' a column is released only if the cells filed since the previous total line are all figures or empty windows and the figures sum exactly to the figure filed on the total line, an empty window adding nothing to the sum as a blank cell adds nothing in Section~III-A of the main text: conditions (G1) and (G2) without the corroboration (G3); the block starts on the line after the previous total line or, if that makes more columns foot, on a later one, the segmentation objective of Section~IV-B. The table extractors go through the comb with the OCR words their cells are built from: img2table's Tesseract words of confidence 50 or more, and the PaddleOCR words SLANet-plus is given, so that its comb row equals PaddleOCR's. Textract's stored output and the typists' sheets carry no word positions and keep their grids. The check also runs on each text reader's lines as returned, the $k$th figure of a line taken as its $k$th cell (``as filed'' in Table~S2). Nothing in the layer was tuned to these panels, with one exception disclosed in the package: the comb was first fitted by one-dimensional $k$-means, and after a smoke test on two of the panels showed scattered numbers from leader dots claiming a window, the rule above replaced it before the full run. The layer's precision is what matters: a released column that is not exact is a false acceptance.

\begin{table*}[t]
\caption{Cells Correct per Panel}
\label{supp:tab:panels}
\centering\footnotesize
\begin{tabular}{@{}lrrrrrrr@{}}
\toprule
Reader & MN p.~340 & MN p.~377 & WI p.~455 & WI p.~652 L & WI p.~652 R & WI p.~653 L & WI p.~653 R\\
\midrule
Tesseract 5.3 (mode 6, best model, $\times2$) & 71.2/78.8 & 68.1/80.1 & 28.9/69.1 & 92.9/96.8 & 93.5/95.2 & 96.4/96.4 & 98.1/98.1\\
EasyOCR 1.7 ($\times2$) & 35.3/61.8 & 25.7/37.7 & 23.4/56.9 & 81.0/88.1 & 45.2/58.9 & 70.0/82.7 & 65.4/75.0\\
PaddleOCR PP-OCRv4 ($\times2$) & 86.5/95.9 & 89.5/96.7 & 43.8/92.1 & 100.0/100.0 & 98.4/98.4 & 99.1/99.1 & 100.0/100.0\\
docTR 1.1 ($\times2$) & 84.7/94.7 & 92.4/96.7 & 50.3/97.7 & 100.0/100.0 & 100.0/100.0 & 100.0/100.0 & 100.0/100.0\\
img2table 2.0 ($\times2$) & 42.4/54.1 & 57.6/59.4 & 44.4/56.6 & 75.4/80.2 & 54.0/57.3 & 0.0/78.2 & 15.4/83.7\\
SLANet-plus ($\times2$) & 35.9/95.9 & 0.0/96.7 & 59.9/92.1 & 62.7/100.0 & 44.4/98.4 & 68.2/99.1 & 22.1/100.0\\
\bottomrule
\end{tabular}
\par\smallskip
\parbox{\textwidth}{\footnotesize Percent of each panel's figures filed in their own row and column, each reader at the configuration of Table~I of the main text (scale $\times1$ or $\times2$; for Tesseract also the page-segmentation mode and model), as filed and, after the slash, with the reader's words filed by TableVision's comb (the values after the arrows in Table~I of the main text; img2table: its Tesseract words; SLANet-plus: the PaddleOCR words it is given).}
\end{table*}

\begin{table*}[t]
\caption{Every Score, Pooled}
\label{supp:tab:all}
\centering\footnotesize
\begin{tabular}{@{}lrrrrrrr@{}}
\toprule
Reader & Chars. & Figs. & Cells & In order & Rows & Cols. & Released\\
\midrule
Tesseract 5.3, as filed & 90.9 & 86.4 & 69.0 & 73.3 & 64.0 & 3/30 & 1 (1)\\
EasyOCR 1.7, as filed & 77.3 & 67.7 & 41.6 & 46.0 & 37.6 & 0/30 & 1 (1)\\
PaddleOCR PP-OCRv4, as filed & 98.1 & 96.5 & 81.4 & 94.1 & 77.6 & 9/30 & 2 (0)\\
docTR 1.1, as filed & 99.0 & 97.9 & 83.7 & 97.4 & 79.2 & 11/30 & 4 (0)\\
img2table 2.0, as filed & 60.5 & 50.7 & 44.8 & 19.9 & 16.8 & 0/30 & 0 (0)\\
SLANet-plus, as filed & 78.5 & 76.9 & 39.1 & 33.6 & 25.1 & 1/30 & 0 (0)\\
Tesseract 5.3, TV & 90.9 & 86.4 & 83.4 & 74.5 & 66.4 & 9/30 & 3 (0)\\
EasyOCR 1.7, TV & 77.3 & 67.7 & 60.5 & 46.0 & 39.2 & 1/30 & 0 (0)\\
PaddleOCR PP-OCRv4, TV & 98.1 & 96.5 & 96.5 & 94.1 & 89.9 & 14/30 & 6 (0)\\
docTR 1.1, TV & 99.0 & 97.9 & 97.9 & 97.4 & 93.6 & 24/30 & 16 (0)\\
img2table 2.0, TV & 77.3 & 64.1 & 63.7 & 37.9 & 34.7 & 0/30 & 1 (1)\\
SLANet-plus, TV & 98.1 & 96.5 & 96.5 & 94.1 & 89.9 & 14/30 & 6 (0)\\
Text layer of p.~455, as filed & 98.1 & 96.7 & 48.0 & 88.5 & 21.9 & 0/14 & 1 (1)\\
Text layer of p.~455, TV & 98.1 & 96.7 & 95.4 & 88.5 & 85.9 & 9/14 & 7 (0)\\
Amazon Textract & $>$99.9 & 99.9 & 99.9 & 99.9 & 99.5 & 33/34 & 18 (0)\\
Typist A & 99.5 & 99.2 & 94.7 & 94.7 & 76.6 & 25/34 & 11 (1)\\
Typist B & 99.9 & 99.9 & 99.9 & 99.9 & 99.0 & 33/34 & 18 (0)\\
\bottomrule
\end{tabular}
\par\smallskip
\parbox{\textwidth}{\footnotesize Pooled over the seven panels (text layer: p.~455; Textract and typists: the four keyed pages), each reader at the configuration of Table~I of the main text. Characters and figures read are position-free; cells and in order are as defined above; TV: the reader's words filed by TableVision's comb, the values after the arrows in Table~I of the main text; released counts the columns whose filed cells sum exactly to the filed total, (G1)--(G2) of the gate, on the output of the row (a text reader's lines, a table reader's grid, or the comb's cells), the wrong ones in parentheses.}
\end{table*}

\emph{What the panels show.} The loss from reading to filing gathers where tables are wide and sparse (Table~S1). On the four 1954 panels, two columns wide with their zeros printed, docTR and PaddleOCR file 98.4--100\% of the figures; on p.~455, seven columns wide with 144 of its 448 cells blank, they file 50.3\% and 43.8\%, although 96.7\% and 84.9\% of the figures stand in the correct order within their rows, so the loss there is mostly blanks that left no token. The comb recovers 97.7\% and 92.1\% on that page. Through the comb, for all six readers, and on the keyed grids, the check releases 79 columns (Table~S2), 32 from comb-filed words and 47 from the keyed grids, and 77 are exact (with the text layer of p.~455, 86 and 84). On the comb-filed words of the five distinct open-source readers (SLANet-plus's being PaddleOCR's), 100 columns close on their panel and 27 are correct: the check releases 25 of the 27 and 1 of the other 73. On the keyed grids, 57 columns close on their page and 47 are correct: it releases 46 of the 47 and 1 of the other 10. The three correct columns it withholds have no line labeled ``Total'' in Tesseract's words (two) or a typist's ``XX'' on a heading row (one), each a gap in the sense of (G1). One exception is typist~A's Schleier column for Waukesha County on Wisconsin 1940 p.~602, where six figures are displaced by one row (the typist's 1, 0, 1, 0, 2, 0 against the reference's 0, 1, 0, 1, 0, 2) and the column still sums to its printed total. The other is Langlade County's Hirst column on p.~455 in img2table's words, which come from Tesseract in mode 11: five figures dropped (9, 4, 24, 4, and 9, together 50) and a 30 read as 80. Over every configuration run, comb-filed words give 59 releases, 57 exact, 47 of 49 on the pages with blank cells (SLANet-plus's comb rows, which repeat PaddleOCR's, not counted); both exceptions are this column, read by Tesseract in mode 11 at $\times2$ alone and inside img2table. Amazon Textract's output also carries one correct figure behind a run of asterisks (Hampton, p.~340), which the parser strips; counted strictly, Textract errs on 2 of the 1356 keyed cells.

\emph{The check on lines as returned.} On the text readers' own lines the check releases 8 columns at the configurations of Table~I, 6 of them exact (with the text layer, 9 and 6). The exact ones are columns in which nothing slid: five on the 1954 panels and Langlade's Lueck column, left of every blank on p.~455. The wrong ones are Lafayette's Quick column on p.~455 in Tesseract's lines and in the text layer, which holds the Scattering votes of Seymour and South Wayne in place of its own votes from New Diggings and Darlington's first ward and still sums to its printed 2, and Langlade's Comings column in EasyOCR's lines, ten figures from three columns that sum to its printed 161. Over every configuration run the lines give 21 releases, 12 exact; five of the nine wrong ones are runs of rows ending at the total line that happen to sum to it, chosen because on misfiled lines the true block foots in fewer columns than a later start, the condition of Section~IV-B for the cut failing.

\section*{Failure Modes}
Table~\ref{supp:tab:failures} collects the failure modes met in development and in the test bed, the signal that exposed each, and what became of the affected data (Section~VI-E of the main text).

\begin{table*}[t]
\caption{Failure Modes, the Signals That Exposed Them, and Their Disposition}
\label{supp:tab:failures}
\centering\scriptsize
\begin{tabular}{@{}rp{6.0cm}p{5.4cm}p{5.3cm}@{}}
\toprule
& Failure mode (where observed) & Signal that exposed it & Disposition\\
\midrule
1 & Sideways pages; skew over half a row pitch (IA 1894, 1895) & rendering; row 20 px lower at the right edge & rotate, deskew; one page, two counties, not attempted\\
2 & Column comb shifted by a 28-px double rule (IA) & width guard (b); a total under another's name & comb re-measured (S5)\\
3 & Data row taken for a heading or never detected (six rows in five counties, IA 1894) & $\mathrm{Omit}_1$: every column short by one row's figures; gap at a label line & row read off the page and block re-gated (four counties)\\
4 & Sub-heading given a reading from the column rule (eight rows, WI 1940) & by-eye check; residuals & readings set absent (S1, S4)\\
5 & Last column clipped by the scan edge (WI 1940; WI 1956) & window against page edge (S6); unread cells & columns withheld\\
6 & Witness misreads: text-layer drift (297 for 207, WI 1940); broken glyphs (42 read as 12, MN 1889) & witness disagreement; residual; 5--8$\times$ re-read & corrected only from the page\\
7 & Whole-county first-pass misreads (18 of 5630 cells, Milwaukee, WI 1940) & by-eye read; residuals & corrected; six columns still off by the book's addition\\
8 & Printed arithmetic that does not close (IA 1896; Langlade, WI 1925; Winnebago, WI 1940; Dane, WI 1958) & rows contradict the total; row dominance; coherence screen; canvass & no cell bent; withheld (IA 1896, Dane) or flagged (Langlade); WI 1940 not released\\
9 & Blind spots of the sum: compensating errors in an accepted column (Rock, WI 1962); sum-preserving displacement (typist A) & official county return only; row-grid comparison & remains \Exact, noted for page inspection; structural constraint\\
\bottomrule
\end{tabular}
\par\smallskip
\parbox{\textwidth}{\footnotesize Observed in development and in the test bed (Section~VI-E of the main text). WI: Wisconsin Blue Book; IA: Iowa Official Register; MN: Minnesota Legislative Manual. Years are volume years (the 1925 volume reports the 1924 primary; the 1954--1962 volumes the elections two years earlier).}
\end{table*}

\section*{The Hard Pages of Table~I}
\emph{Pages.} The bottom panel of Table~I of the main text is three right-hand pages of the 1889 Iowa Official Register, which reports the November 1888 election: pp.~113, 145, and 171, embedded in the project's scan as one-bit images at 118 pixels per inch, about $680\times960$ pixels a page. Every right-hand page of the volume's returns continues the rows of its facing page with 16 more candidate columns, Treasurer of State (four), Judge of the Supreme Court (three), Attorney General (three), and Railroad Commissioners (six), under rotated names, and prints no row labels. Each of the three holds two county blocks: Fayette, closing on p.~113, and the first 12 rows of Floyd, which closes on p.~115; Lyon and Madison; and the last 14 rows of Scott, whose total also covers 8 rows on p.~169, and Shelby. They are the volume's continuation pages for which the workflow's reader has stored output, made in June 2026 by the project's re-reading of problem pages, before the benchmark existed. For p.~171 the project's pipeline stored two passes, the second made because Scott's columns, whose totals cover rows on p.~169, failed its tolerance test; it keeps the first when the second does no better, as here, and the benchmark scores the pass it kept. The second differs in one cell, reading Pleasant Valley's Mahin as the reference's 88; scored instead, it would add one cell and one county column and change no release. Amazon Textract's stored table output reads 54 of the volume's 58 continuation pages sideways, these three among them: each of its grids has 16 rows, one for each of the page's columns, holding that column's figures fused into a few cells, read from the bottom up and some of them upside down, and 3 to 11 columns of such cells and header fragments.

\emph{Reference.} The cells were read for this benchmark from the page images at four times (Lanczos), every cell of a column that did not foot re-read at 7 to 14 times (nearest neighbor), and checked against the printed totals. All 16 columns of Fayette and of Shelby, the 13 of Lyon's that print votes, and 14 of Madison's reproduce their printed totals exactly. Madison's Campbell and Mahin columns do not close (cells 2296 and 1804, printed 2298 and 1803); every cell of both was re-read at pixel level, the reader's pass reads them identically, and the printed totals are kept. Floyd's 12 rows with the 4 on p.~115, and Scott's 14 with the 8 on p.~169, reproduce those counties' 16 printed totals each. One first reading was corrected by this check, Scott's Campbell at Rockingham, 21, whose last glyph has the descender of a 4. Compared afterwards with the reader's pass, the reference differs in 12 of the 1616 cells (Table~\ref{supp:tab:disputes}), and in each it is the reference value that makes its column reproduce its printed total; a crop of each is in the package. The reader's output for p.~113 had been displayed once, while the data were being located, before the reference was read. The reference was read by an AI model of the reader's family (Claude), not by people; it is checked against the printed totals, not keyed by independent typists, so a misreading that leaves a column sum intact would escape it. A second model of the same family, reading 40 randomly chosen cells, 8 more of Shelby's, and every cell of Madison's two non-closing columns before seeing the reference, agreed in all 86. The panels hold 1494 figures (1417 cells and 77 printed totals) and 77 county columns with a printed total on the page, 59 of which close there.

\emph{Protocol.} Each county block is a panel, the top block with the page's header. Because a continuation page has no row labels, for a block that ends in a printed total the check takes the reader's last line holding figures as the total line, the crop ending at that line, for every reader alike; Floyd's block, which ends without one, can release nothing. The open-source readers get each crop at one, two, three, and four times (Lanczos), and the configuration that files the most cells, pooled over the six panels, is reported as before. Textract and the workflow's reader have stored whole-page output and are scored on whole pages against the same reference. The reader is Claude Sonnet~4.6 through the project's transcription script, one full-page image at 1568 pixels on the long edge, prompted to return every table as JSON with its rows and printed totals; its rows are filed as returned, the row numbers it supplies for unlabeled rows dropped, and its totals as a line labeled ``Total,'' on which the check runs unchanged. Neither Textract nor the workflow's reader was run again for the benchmark.

\emph{What the hard pages show.} Recognition fails before structure can help (Table~\ref{supp:tab:hard}): at its reported configuration the best open-source reader reproduces 71.0\% of the digits, files 19.0\% of the figures from its own lines and 65.8\% through the comb, and gets no county column correct, and the table extractors do worse; over every configuration run, two county columns come out correct, both through the comb, and neither is released. Textract's sideways grid holds 76.7\% of the digits, but only 1 of the 1494 figures in its cell. The reader's pass files 99.2\%; its 12 wrong cells are seven misread figures (14 for 12, 89 for 92, 84 for 82, 56 for 36, 64 for 84, 83 for 88, and 43 for 45), two figures shifted one cell sideways, each overwriting its neighbor's figure and leaving its own cell blank, and one dropped figure; the 8 closing columns they fall in fail the check, which releases the other 51 closing columns, all exact. Over every configuration run, the open-source readers' output yields 23 releases, 12 from their lines and 11 through the comb (SLANet-plus's comb rows, which repeat PaddleOCR's words, not counted), and none is exact: runs of misread or misfiled figures whose sum equals a figure on the total line, either a small or misread one (2, 7, 37, 96, or 531) or a printed total reached with figures from other columns; 16 of the 23 rest on a figure that is not the column's printed total.

\begin{table*}[t]
\caption{The Hard Pages, Every Reader}
\label{supp:tab:hard}
\centering\footnotesize
\begin{tabular}{@{}lrrrrrrr@{}}
\toprule
Reader & Chars. & Figs. & Cells & In order & Rows & Cols. & Released\\
\midrule
Tesseract 5.3 (mode 6, best model, $\times2$) & 67.4 & 63.3 & 15.0 & 19.7 & 0.0 & 0/77 & 0 (0)\\
EasyOCR 1.7 ($\times3$) & 29.9 & 24.2 & 3.5 & 4.1 & 0.0 & 0/77 & 0 (0)\\
PaddleOCR PP-OCRv4 ($\times4$) & 60.8 & 56.2 & 13.9 & 16.3 & 0.0 & 0/77 & 0 (0)\\
docTR 1.1 ($\times1$) & 71.0 & 69.5 & 19.0 & 30.7 & 0.0 & 0/77 & 1 (1)\\
img2table 2.0 ($\times3$) & 18.6 & 17.6 & 3.8 & 1.1 & 0.0 & 0/77 & 0 (0)\\
SLANet-plus ($\times1$) & 32.8 & 30.5 & 14.8 & 3.9 & 0.9 & 0/77 & 0 (0)\\
Tesseract 5.3, TV (mode 6, best model, $\times2$) & 67.4 & 63.3 & 53.3 & 18.8 & 0.0 & 0/77 & 1 (1)\\
EasyOCR 1.7, TV ($\times3$) & 29.9 & 24.2 & 12.7 & 3.8 & 0.0 & 0/77 & 0 (0)\\
PaddleOCR PP-OCRv4, TV ($\times4$) & 60.8 & 56.2 & 51.3 & 16.2 & 0.0 & 0/77 & 0 (0)\\
docTR 1.1, TV ($\times1$) & 71.0 & 69.5 & 65.8 & 29.9 & 3.8 & 0/77 & 0 (0)\\
img2table 2.0, TV ($\times3$) & 54.1 & 49.7 & 37.1 & 11.8 & 0.0 & 0/77 & 0 (0)\\
SLANet-plus, TV ($\times1$) & 34.8 & 32.2 & 28.7 & 6.3 & 0.9 & 0/77 & 0 (0)\\
Amazon Textract (stored output) & 76.7 & 73.1 & 0.1 & 0.0 & 0.0 & 0/77 & 0 (0)\\
TableVision reader (one stored pass) & 99.7 & 99.3 & 99.2 & 97.4 & 90.6 & 67/77 & 51 (0)\\
\bottomrule
\end{tabular}
\par\smallskip
\parbox{\textwidth}{\footnotesize Pooled over the six county blocks of pp.~113, 145, and 171 of the 1889 Iowa register (Textract and the workflow's reader: the three whole pages), each reader at the configuration that files the most cells from its own output (scale $\times1$ to $\times4$; for Tesseract also the page-segmentation mode and model). Columns as in Table~S2; released (wrong).}
\end{table*}

\begin{table}[t]
\caption{Cells Where the Reference and the Reader's Pass Differ}
\label{supp:tab:disputes}
\centering\scriptsize
\setlength{\tabcolsep}{2.5pt}
\begin{tabular}{@{}rlllrrrrr@{}}
\toprule
&&&& \multicolumn{2}{c}{Cell} & & \multicolumn{2}{c}{Column sum}\\
\cmidrule(lr){5-6}\cmidrule(l){8-9}
p. & County & Row & Column & ref. & pass & Total & ref. & pass\\
\midrule
113 & Fayette & Bethel & Jones & 12 & 14 & 289 & 289 & 291\\
113 & Fayette & Fremont & Stone & 92 & 89 & 2815 & 2815 & 2812\\
113 & Fayette & West Union & Rice & 6 & blank & 288 & 288 & 282\\
113 & Fayette & West Union & Sharp & 5 & 6 & 43 & 43 & 44\\
113 & Floyd & 3d ward & Stone & 82 & 84 & 2002 & 2002 & 2004\\
113 & Floyd & Union & Sharp & 1 & blank & 12 & 12 & 11\\
145 & Lyon & Garfield & Twombly & 36 & 56 & 953 & 953 & 973\\
145 & Madison & Lee & Case & 84 & 64 & 1347 & 1347 & 1327\\
145 & Madison & Penn & Smyth & 66 & 29 & 1354 & 1354 & 1317\\
145 & Madison & Penn & Jones & 29 & blank & 376 & 376 & 347\\
171 & Scott & Pleasant Valley & Mahin & 88 & 83 & 2250 & 2250 & 2245\\
171 & Scott & Rockingham & Smyth & 45 & 43 & 5725 & 5725 & 5723\\
\bottomrule
\end{tabular}
\par\smallskip
\parbox{\columnwidth}{\footnotesize Row labels from the facing page. Floyd's and Scott's totals are their county totals, on pp.~115 and 171, with the rows of pp.~115 and 169 included in the sums. In two cells, Lee's Case and Bethel's Jones, the glyph at pixel level looks more like the pass's figure (Textract also returns 64 for Lee's Case); there the reference follows the printed total and the sister columns (Smyth 84, Mitchell 82; Rice and Williamson 12), as the typed pages do for Waterford's broken 42.}
\end{table}

\section*{Proofs of the Results of Section V}
This supplement restates each result of Section~V of the main text, numbered as in the main text, and proves it; two results appear only here, Theorem~4, on when the segmentation recursion returns the true cut (stated in words in Section~IV-B of the main text), and Lemma~4, on the likelihood fitted in Section~VI-C. Every statement is also checked numerically, exhaustively where the domain is finite, by \texttt{replication/theory/theory\_checks.py} in the package. The notation is that of Sections~III--V of the main text: $v_c$ is the vector of printed figures of a column over the county's $R$ constituent lines, $t_c$ the printed total, $\hat V_{ic}$ and $\hat t_c$ the readings, $d_c=\sum_i\hat V_{ic}-\hat t_c$ the residual (Eq.~(1) of the main text), $e_{ic}$ and $e_{tc}$ the reading errors, $b_c$ the source discrepancy, $\sigma_c$ the scope error, $\rho$ the extraction map, and $\mathrm{Cmp}_c$ the completeness predicate. Eq.~(2) of the main text is the identity $d_c=b_c+\sigma_c+\sum_ie_{ic}-e_{tc}$, which holds whenever $\mathrm{Cmp}_c$ does and follows by substituting $\hat V_{ic}=v_{\rho(i)c}+e_{ic}$ and $\hat t_c=t_c+e_{tc}$ into Eq.~(1). In Section~V-A of the main text a correctly scoped column is fixed, one for which $\rho$ is a bijection of $\{1,\dots,\hat R\}$ onto the county's constituent lines (so $\hat R=R$ and $\sigma_c=0$), with consistent printed arithmetic ($b_c=0$), and the readings are written $v_c+e$, $t_c+e_t$. In Section~V-C, cells are defective independently with probability $\theta$, a numeric defect changes a cell by a nonzero amount drawn, independently of everything else, from a law $q$ on $\mathbb Z\setminus\{0\}$, $\theta_e\le\theta$ is the rate of numeric defects, and
\begin{equation}
 \kappa=\max_{k\ge2}\Pr\Big[\textstyle\sum_{i=1}^{k}\varepsilon_i=0\Big]
\label{supp:eq:kappa}
\end{equation}
is the cancellation probability for independent $\varepsilon_i$ with law $q$; the total is assumed read correctly and to be the canvass figure, and $\mathrm{Cred}_c$ is assumed to hold. Throughout, $S_k=\varepsilon_1+\dots+\varepsilon_k$ and $p_k=\Pr[S_k=0]$, so $p_0=1$, $p_1=0$, and $\kappa=\max_{k\ge2}p_k$.

\begin{lemma}[Exactness and complexity; Lemma~1 of the main text]
$F(k,j)$ of Eq.~(5) of the main text is the maximum score of a cut of lines $1,\dots,j$ into $k$ blocks of lengths in $[2,L]$, so the recursion returns a global maximizer over all such $K$-cuts, in $O(NL(C+K))$ time and $O(NL+KN)$ space beyond the $O(NC)$ readings and their prefix sums.
\end{lemma}
\begin{IEEEproof}
The score of a cut is the sum of its block scores, and the last block of an optimal $k$-cut of $1,\dots,j$ is some $[i+1,j]$ with $j-L\le i\le j-2$ whose prefix is an optimal $(k-1)$-cut of $1,\dots,i$; induction on $k$ gives Eq.~(5) of the main text. Prefix sums make each $\phi(i,j)$ an $O(C)$ evaluation over at most $NL$ pairs, and the recursion fills $KN$ states with $O(L)$ constant-time transitions each. The argument uses only that the score of a cut is a sum of block scores, so it holds for any such score, in particular the weighted score of Theorem~2.
\end{IEEEproof}

\begin{proposition}[Kernel]
\label{supp:prop:kernel}
Condition (G2) holds if and only if $\mathbf 1^{\mathsf T}e=e_t$. The undetectable patterns therefore form the hyperplane $H=\{(e,e_t):\mathbf 1^{\mathsf T}e=e_t\}$, a sublattice of rank $R$ in $\mathbb{Z}^{R+1}$: every pattern with exactly one nonzero component is detected, and every permutation of the entries of $v_c$ is undetectable.
\end{proposition}
\begin{IEEEproof}
By Eq.~(2) of the main text with $b_c=\sigma_c=0$, $d_c=\mathbf 1^{\mathsf T}e-e_t$, so (G2) is the single linear equation $\mathbf 1^{\mathsf T}e=e_t$, whose integer solutions form a sublattice of rank $R$. A pattern with one nonzero component violates it; a permutation of $v_c$ has $e=\pi v_c-v_c$ with $\mathbf 1^{\mathsf T}e=0=e_t$.
\end{IEEEproof}

\begin{corollary}[Second constraints]
\label{supp:cor:subtotal}
Partition the lines into groups $g\in\mathcal G$ with separately printed subtotals, read with errors $e_{s_g}$, and keep the county total. A pattern $(e,(e_{s_g})_g,e_t)$ is undetectable if and only if $\mathbf 1_g^{\mathsf T}e=e_{s_g}$ for every $g$ and $\mathbf 1^{\mathsf T}e=e_t$, so the undetectable set has codimension $|\mathcal G|+1$ in $\mathbb{Z}^{R+|\mathcal G|+1}$ instead of one; with the subtotals read correctly, a permutation is detected unless it preserves every group sum, and a compensating pair unless both members share a group.
\end{corollary}
\begin{IEEEproof}
Each printed sum is a residual of the form of Eq.~(2) of the main text over its own lines, and the $|\mathcal G|+1$ linear forms are independent because each contains a variable ($e_{s_g}$ or $e_t$) that occurs in no other form. A permutation has $\mathbf 1^{\mathsf T}e=0$ and is caught by any group whose sum it changes; a compensating pair $e_i=-e_j$ changes the sum of every group containing exactly one of $i,j$.
\end{IEEEproof}

\begin{proposition}[Arbitration by arithmetic]
\label{supp:prop:arbitration}
Let two witnesses read the column as $x^{(1)}$ and $x^{(2)}$, disagreeing on $D=\{i:x^{(1)}_i\ne x^{(2)}_i\}$ with $\delta_i=x^{(2)}_i-x^{(1)}_i$. Assume $\sigma_c=b_c=0$, that the total is read correctly ($e_t=0$), and that for every $i$ at least one witness has the printed figure. Then the true column is $x^{(1)}$ with $x^{(2)}$ substituted on some $T\subseteq D$, and $T$ satisfies $\sum_{i\in T}\delta_i=t_c-\mathbf 1^{\mathsf T}x^{(1)}$. The sum identifies $T$, and hence the column, if and only if that residual has a unique representation as a subset sum of $\{\delta_i\}_{i\in D}$. This holds for every residual when the $|\delta_i|$, sorted increasingly, are superincreasing (each exceeds the sum of all smaller ones), and in particular whenever $|D|=1$.
\end{proposition}
\begin{IEEEproof}
The claim that the true column is $x^{(1)}$ with $x^{(2)}$ substituted on some $T\subseteq D$ restates the hypotheses, $T$ being the set of positions at which $x^{(1)}$ is wrong. Two subsets $T\ne T'$ with equal sums produce two footing columns, and conversely; hence uniqueness of the representation is necessary and sufficient. If the absolute differences are superincreasing, two distinct subsets differ at some largest index $j$, and $|\delta_j|$ exceeds the sum of all smaller absolute differences, so the two sums differ.
\end{IEEEproof}

\begin{proposition}[Signature of omitted rows]
\label{supp:prop:signature}
If a set $\Omega$ of constituent lines is omitted, none duplicated, and all other terms of Eq.~(2) of the main text vanish in every column, the vector of residuals across the county's columns is $\boldsymbol d=-\sum_{\ell\in\Omega}\boldsymbol v_\ell$. A single omitted line is detected if and only if it prints a nonzero figure in some column, and then its figures are determined; which printed line carries them is not, and must be read from the page. From $\boldsymbol d$ alone, a single omitted line valued $-\boldsymbol d$, any set $\Omega$ with $\sum_{\ell\in\Omega}\boldsymbol v_\ell=-\boldsymbol d$, and reading errors with $\sum_ie_{ic}-e_{tc}=d_c$ in every column are observationally equivalent.
\end{proposition}
\begin{IEEEproof}
$\sigma_c=-\sum_{\ell\in\Omega}v_{\ell c}$ for every $c$ and the other terms are zero, so $\boldsymbol d=-\sum_{\ell\in\Omega}\boldsymbol v_\ell$. For a single omitted line $\ell$, $\boldsymbol d=-\boldsymbol v_\ell$ is nonzero exactly when $\ell$ prints a nonzero figure in some column, and it determines $\boldsymbol v_\ell$ but not $\ell$. The equivalence statement is the observation that the three hypotheses produce the same $\boldsymbol d$.
\end{IEEEproof}

\begin{lemma}[Place-value syndromes]
\label{supp:lem:syndrome}
Let a correctly scoped column have exactly one wrong number, a cell or the total, misread or misprinted. (i)~If a digit $a$ at place $p$ is read or printed as $a'$, then $d_c=\pm(a'-a)10^p$ (plus for a cell, minus for the total), so $|d_c|$ has exactly one nonzero digit, which gives $|a'-a|$ and $p$. (ii)~If the number's digits appear in another order, $9\mid d_c$; adjacent digits $b,a$ at places $p+1,p$ appearing as $a,b$ give $d_c=\pm9(a-b)10^p$. (iii)~Two digit substitutions alone in a column cancel if and only if they fall at the same place and move $d_c$ by opposite amounts.
\end{lemma}
\begin{IEEEproof}
With $\sigma_c=0$, Eq.~(2) of the main text reads $d_c=b_c+\sum_ie_{ic}-e_{tc}$. If the one wrong number is a cell that is off by $e$, then $d_c=e$: a misread cell has $e_{ic}=e$ and every other term zero, and a misprinted cell, the readings being correct and the total being the sum the cells require, has $b_c=e$. If it is the total, off by $e$, then $d_c=-e$ in the same two ways. A number with digits $a_p$ is $\sum_pa_p10^p$. (i)~Changing the digit at place $p$ from $a$ to $a'$ changes the number by $e=(a'-a)10^p$, with $1\le|a'-a|\le9$; the decimal expansion of $|d_c|=|a'-a|10^p$ is the single digit $|a'-a|$ followed by $p$ zeros, so it names $|a'-a|$ and $p$, and distinct pairs give distinct values. (ii)~Since $10\equiv1\pmod9$, every number is congruent modulo~9 to the sum of its digits, which a reordering preserves; so $9\mid e$. Reading $b$ at place $p+1$ and $a$ at place $p$ as $a$ and $b$ changes the number by $(a10^{p+1}+b10^p)-(b10^{p+1}+a10^p)=(a-b)(10^{p+1}-10^p)=9(a-b)10^p$. (iii)~The two substitutions add $u_110^{p_1}+u_210^{p_2}$ to $d_c$, where $u_i$ is the digit change, with sign reversed if it falls in the total, so $1\le|u_i|\le9$; this covers two substitutions in one number, at two places of it. If $p_1<p_2$, then $|u_210^{p_2}|\ge10^{p_2}\ge10\cdot10^{p_1}>|u_1|10^{p_1}$, so the sum is not zero; if $p_1=p_2$, it is zero exactly when $u_1=-u_2$.
\end{IEEEproof}

\begin{proposition}[One among $R+1$]
\label{supp:prop:which}
Suppose each of the $R+1$ printed numbers of a column, the $R$ cells and the total, is defective independently, with probability $\beta_b$ for a cell and $\beta_t$ for the total. Conditional on the column containing exactly one defective number, the probability that it is the total is
\begin{equation}
 \frac{\beta_t(1-\beta_b)}{\beta_t(1-\beta_b)+R\beta_b(1-\beta_t)}\approx\frac{\beta_t}{\beta_t+R\beta_b},
 \label{supp:eq:which}
\end{equation}
which is $1/(R+1)$ when $\beta_t=\beta_b$. Over failed columns with lengths $R_c$ and one defect each at equal rates, the expected number whose total is the defective figure is $\sum_c1/(R_c+1)$.
\end{proposition}
\begin{IEEEproof}
The event that only the total is defective has probability $\beta_t(1-\beta_b)^R$ and the event that only one cell is defective $R\beta_b(1-\beta_t)(1-\beta_b)^{R-1}$; Bayes' rule gives Eq.~\eqref{supp:eq:which}, and the approximation drops the factors $1-\beta$. Equal rates make the defective position uniform over $R+1$ exchangeable positions.
\end{IEEEproof}

\begin{lemma}[Cancellation]
\label{supp:lem:kappa}
For every law $q$, $\kappa\le\tfrac12$. If $q$ is symmetric, $\kappa=\Pr[\varepsilon_1+\varepsilon_2=0]=\sum_mq(m)^2$, and the maximum in Eq.~\eqref{supp:eq:kappa} is attained at $k=2$.
\end{lemma}
\begin{IEEEproof}
Write $M=\max_mq(m)$. Conditioning on the first $r-1$ summands gives, for every integer $u$,
\begin{equation}
 \Pr[S_r=u]=\sum_{s}\Pr[S_{r-1}=s]\,q(u-s)\le M.
 \label{supp:eq:cond}
\end{equation}
If $M\le\frac12$, Eq.~\eqref{supp:eq:cond} with $u=0$ gives $\Pr[S_k=0]\le\frac12$ for every $k$. If $M=1$, then $S_k=km_0\ne0$ and $\Pr[S_k=0]=0$. If $\frac12<M<1$, the atom $m_0$ with $q(m_0)=M$ is unique, every other atom has mass at most $1-M$, and splitting the sum in Eq.~\eqref{supp:eq:cond} at $s=u-m_0$ gives
\begin{align}
 \Pr[S_r=u]&\le M\Pr[S_{r-1}=u-m_0]\notag\\
 &\quad+(1-M)\big(1-\Pr[S_{r-1}=u-m_0]\big)\notag\\
 &=(1-M)+(2M-1)\Pr[S_{r-1}=u-m_0].
\end{align}
Apply this with $u=-(k-r)m_0$ for $r=k,k-1,\dots,1$, so that the target $u-m_0=-(k-r+1)m_0$ is the one required at the next step, and end with $\Pr[S_0=-km_0]=0$ since $S_0=0$ and $m_0\ne0$: with $\varrho=2M-1\in(0,1)$ and $(1-M)/(1-\varrho)=\tfrac12$,
\begin{equation}
 \Pr[S_k=0]\le(1-M)\sum_{i=0}^{k-1}\varrho^{\,i}=\frac{1-\varrho^{k}}{2}<\frac12 .
\end{equation}
So $\kappa\le\frac12$ for every $q$. If $q$ is symmetric, its characteristic function $\varphi(t)=\sum_mq(m)\mathrm e^{\mathrm itm}$ is real with $|\varphi|\le1$, so $\varphi(t)^k\le\varphi(t)^2$ pointwise for every $k\ge2$ (for odd $k$, $\varphi^k\le|\varphi|^k\le\varphi^2$), and $\Pr[S_k=0]=\frac1{2\pi}\int_{-\pi}^{\pi}\varphi(t)^k\,dt\le\frac1{2\pi}\int_{-\pi}^{\pi}\varphi(t)^2\,dt=\Pr[S_2=0]=\sum_mq(m)q(-m)=\sum_mq(m)^2$.
\end{IEEEproof}

\begin{corollary}[Place and digit]
\label{supp:cor:place}
If every numeric defect is one substitution, at place $p$ with probability $w_p$ and changing the cell by $\delta10^p$ with $\delta$ drawn from a symmetric law $P$ independent of $p$, then $\kappa=\big(\sum_pw_p^2\big)\big(\sum_\delta P(\delta)^2\big)$: the collision entropy $-\log\kappa$ of the error law is the sum of those of place and digit. A uniform digit misread as a uniformly chosen other digit has $\sum_\delta P(\delta)^2=19/270$.
\end{corollary}
\begin{IEEEproof}
By Lemma~\ref{supp:lem:syndrome}(i), the map $(p,\delta)\mapsto\delta10^p$ is one-to-one on $\{0,1,\dots\}\times\{\pm1,\dots,\pm9\}$, so the law of a numeric defect is $q(\delta10^p)=w_pP(\delta)$, which is symmetric because $P$ is. Lemma~\ref{supp:lem:kappa} gives $\kappa=\sum_mq(m)^2=\sum_p\sum_\delta w_p^2P(\delta)^2$, which factors; taking $-\log$, the collision (R\'enyi order-two) entropy $-\log\sum_mq(m)^2$ is the sum of those of $w$ and $P$. If a uniform digit $a$ is replaced by a uniformly chosen $a'\ne a$, then $P(\delta)=|\{(a,a'):a'-a=\delta\}|/90=(10-|\delta|)/90$ for $1\le|\delta|\le9$, and $\sum_\delta P(\delta)^2=2\sum_{j=1}^{9}j^2/8100=570/8100=19/270$.
\end{IEEEproof}

\begin{proposition}[Detection and false acceptance]
\label{supp:prop:detect}
For a column of $R$ cells,
\begin{equation}
 1-(1-\theta)^R-\tbinom R2\theta_e^2\kappa\;\le\;\Pr[\text{fail}]\;\le\;1-(1-\theta)^R,
 \label{supp:eq:detect}
\end{equation}
and the probability that a column passes while containing a defect is at most $\binom R2\theta_e^2\kappa\le(R\theta_e)^2\kappa/2$, or $\binom R2\theta_e^2\kappa/(1-\theta)^R$ conditional on passing.
\end{proposition}
\begin{IEEEproof}
A defect-free column passes, because its cells sum to the correctly read total and $\mathrm{Cred}_c$ holds by assumption; so $\Pr[\text{fail}]\le\Pr[\text{some defect}]=1-(1-\theta)^R$. A column with an unreadable cell fails (G1), and a column with exactly one numeric defect and no other fails (G2), since a single nonzero term does not cancel; so a defective column passes only if it has at least two numeric defects whose sum vanishes. The probability of at least two numeric defects is at most $\binom R2\theta_e^2$ by the union bound over pairs, and conditional on any $k\ge2$ of them the sum vanishes with probability at most $\kappa$; hence $\Pr[\text{defect}\wedge\text{pass}]\le\binom R2\theta_e^2\kappa$, which gives the lower bound in Eq.~\eqref{supp:eq:detect}. Divide by $\Pr[\text{pass}]\ge(1-\theta)^R$ for the conditional form.
\end{IEEEproof}

\begin{theorem}[Silent error]
\label{supp:thm:silent}
Let $p_2=\sum_mq(m)q(-m)$ and $\rho=\theta_e/(1-\theta)$. For every $R$, the silent error rate $s=\mathbb E[Z\,\mathbf 1\{\text{pass}\}]/(R\Pr[\text{pass}])$ of the cells of \Exact\ columns, $Z$ the number of wrong cells, satisfies
\begin{equation}
 (R-1)\,p_2\,\theta_e^{2}(1-\theta)^{R-2}\;\le\;s\;\le\;(R-1)\,\kappa\,\rho^{2}(1+\rho)^{R-2}.
 \label{supp:eq:silent}
\end{equation}
If $q$ is symmetric, $p_2=\kappa$, the bounds differ by a factor of at most $(1-\theta)^{-2(R-1)}$, and $s=(R-1)\kappa\theta_e^2\,(1+O(R\theta))$ for $R\theta\le1$.
\end{theorem}
\begin{IEEEproof}
Let $N_e$ and $N_u$ count the cells with numeric and with unreadable defects. A column passes exactly when $N_u=0$ and the $N_e$ errors sum to zero, and on that event $Z=N_e$. Given $N_e=k$ and $N_u=0$, the errors are independent with law $q$, and $\Pr[N_e=k,N_u=0]=\binom Rk\theta_e^k(1-\theta)^{R-k}$, since each of the other $R-k$ cells is defect-free with probability $1-\theta$. Hence
\begin{align}
 \mathbb E[Z\,\mathbf 1\{\text{pass}\}]&=\sum_{k\ge2}k\tbinom Rk\theta_e^k(1-\theta)^{R-k}p_k,\label{supp:eq:EZ}\\
 \Pr[\text{pass}]&=\sum_{k\ge0}\tbinom Rk\theta_e^k(1-\theta)^{R-k}p_k\in\big[(1-\theta)^R,1\big].\notag
\end{align}
For the lower bound keep the term $k=2$ of Eq.~\eqref{supp:eq:EZ} and use $\Pr[\text{pass}]\le1$: $s\ge2\binom R2\theta_e^2(1-\theta)^{R-2}p_2/R=(R-1)p_2\theta_e^2(1-\theta)^{R-2}$. For the upper bound use $p_k\le\kappa$ for $k\ge2$ and $\Pr[\text{pass}]\ge(1-\theta)^R$: $s\le\kappa R^{-1}\sum_{k\ge2}k\binom Rk\rho^k$. Since $\sum_{k\ge1}k\binom Rk\rho^k=R\rho(1+\rho)^{R-1}$, the sum from $k=2$ is $R\rho\big((1+\rho)^{R-1}-1\big)\le R\rho\cdot(R-1)\rho(1+\rho)^{R-2}$ by the mean value theorem, which gives the right side of Eq.~\eqref{supp:eq:silent}. If $q$ is symmetric, $p_2=\sum_mq(m)^2=\kappa$ by Lemma~\ref{supp:lem:kappa}, and the ratio of the bounds is $(1-\theta)^{-R}(1+\rho)^{R-2}\le(1-\theta)^{-2(R-1)}$, because $1+\rho=(1-\theta+\theta_e)/(1-\theta)\le(1-\theta)^{-1}$; as $(1-\theta)^{-2(R-1)}=1+O(R\theta)$ when $R\theta$ is bounded, both bounds equal $(R-1)\kappa\theta_e^2(1+O(R\theta))$.
\end{IEEEproof}

\begin{proposition}[Repair by arithmetic]
\label{supp:prop:repair}
Let one cell of a column be unreadable, the other $R-1$ carry numeric defects independently with probability $\theta_e$, and the total be misread with probability $\pi_t$. Writing into the unreadable cell the value that makes the column foot satisfies (G2) whatever the errors, and the written value is wrong exactly when the other $R$ numbers carry a nonzero net error; the expected share of wrong cells lies between $(R-1)\theta_e\big(1+(1-\theta_e)^{R-2}(1-\pi_t)\big)/R$ and $\big(2(R-1)\theta_e+\pi_t\big)/R$.
\end{proposition}
\begin{IEEEproof}
Let $i$ be the unreadable cell, with $b_c=\sigma_c=0$, so $\hat t_c=\sum_jv_j+e_t$ and $\hat V_j=v_j+e_j$. The written value is $w=\hat t_c-\sum_{j\ne i}\hat V_j=v_i+N$ with $N=e_t-\sum_{j\ne i}e_j$, and the repaired column's residual is $\sum_{j\ne i}\hat V_j+w-\hat t_c=0$ identically, so (G2) holds whatever the errors, and $w$ is wrong exactly when $N\ne0$. The number of wrong cells is $Z=|\{j\ne i:e_j\ne0\}|+\mathbf 1\{N\ne0\}$, so $\mathbb E[Z]=(R-1)\theta_e+\Pr[N\ne0]$. The event $N\ne0$ requires an error among the other $R$ numbers, so $\Pr[N\ne0]\le(R-1)\theta_e+\pi_t$ by the union bound; and if exactly one of them is in error, $N$ is that error up to sign and is not zero, so $\Pr[N\ne0]\ge(R-1)\theta_e(1-\theta_e)^{R-2}(1-\pi_t)$. Divide by $R$.
\end{IEEEproof}

\emph{The model of Theorem~2.} Lines $1,\dots,N$ carry readings $x_{rc}$; an admissible $K$-cut $\mathbf j=(j_1<\dots<j_K=N)$, $j_0=0$, has blocks $B_k=[j_{k-1}+1,j_k]$ of lengths in $[2,L]$, body lines $j_{k-1}+1,\dots,j_k-1$, total line $j_k$, and body column sums $s_{kc}=\sum_{r=j_{k-1}+1}^{j_k-1}x_{rc}$; $B_k$ foots in $c$ when $x_{j_kc}=s_{kc}$, the indicator of Eq.~(4) of the main text when no reading is $\bot$. Under $\mathcal M$, given $\mathbf j$: every body-line reading $x_{rc}$ is drawn from a law $f_c$ on $\mathbb Z_{\ge0}$; each total-line reading equals $s_{kc}$ with probability $1-\varepsilon_c$, $0<\varepsilon_c<1$, and is otherwise drawn from $f_c$, like a constituent figure; all draws are independent, and $f_c$ is positive at every reading that occurs.

\begin{theorem}[Segmentation as inference]
\label{supp:thm:infer}
Under $\mathcal M$, for every admissible $K$-cut $\mathbf j$,
\begin{equation}
 \log\Pr[x\mid\mathbf j]=\mathrm{const}(x)+\sum_{k=1}^{K}\sum_{c=1}^{C}\omega_c(x_{j_kc})\,\mathbf 1\{B_k\text{ foots in }c\},
 \label{supp:eq:loglik}
\end{equation}
where $\omega_c(u)=\log\big(1+(1-\varepsilon_c)/(\varepsilon_cf_c(u))\big)$. (i)~Eq.~(5) of the main text with block score $\sum_c\omega_c(x_{jc})[\text{foot}_c]$ returns a maximum-likelihood cut; if the $\varepsilon_c$ are equal and the $f_c$ uniform on ranges of one length holding every reading, that score is a constant multiple of $\phi$. (ii)~Under a uniform prior on admissible cuts, Eq.~(5) with $(\max,+)$ replaced by $(+,\times)$, run forward and backward, gives the posterior probability that line $j$ closes block $k$, for all $j$ and $k$, in the time of Lemma~1. (iii)~A block whose last line is not a total line foots in column $c$ with probability at most $\max_uf_c(u)$, independently across columns.
\end{theorem}
\begin{IEEEproof}
By the definition of $\mathcal M$,
\begin{multline*}
 \Pr[x\mid\mathbf j]=\prod_{r\notin\{j_k\}}\prod_cf_c(x_{rc})\\
 \times\prod_{k,c}\Big((1-\varepsilon_c)\mathbf 1\{x_{j_kc}=s_{kc}\}+\varepsilon_cf_c(x_{j_kc})\Big).
\end{multline*}
Multiply and divide each factor of the last product by $f_c(x_{j_kc})>0$. With $\Phi(x)=\prod_{r=1}^{N}\prod_cf_c(x_{rc})$, which does not depend on $\mathbf j$,
\begin{multline*}
 \Pr[x\mid\mathbf j]=\Phi(x)\Big(\prod_c\varepsilon_c\Big)^{K}\\
 \times\prod_{k,c}\Big(1+\frac{1-\varepsilon_c}{\varepsilon_cf_c(x_{j_kc})}\mathbf 1\{x_{j_kc}=s_{kc}\}\Big),
\end{multline*}
and taking logarithms gives Eq.~\eqref{supp:eq:loglik} with $\mathrm{const}(x)=\log\Phi(x)+K\sum_c\log\varepsilon_c$, the same for every $K$-cut. (i)~The log-likelihood is a constant plus a sum of block scores, so Lemma~1 applies to the weighted score; with $\varepsilon_c=\varepsilon$ and $f_c\equiv1/W$ on the readings, $\omega_c\equiv\log(1+(1-\varepsilon)W/\varepsilon)$ and the score is that constant times $\phi$. (ii)~With a uniform prior, $\Pr[\mathbf j\mid x]\propto\exp\sum_kS_\omega(B_k)$, $S_\omega$ the weighted block score. Let $A(k,j)$ be the sum of $\exp\sum S_\omega$ over admissible cuts of lines $1,\dots,j$ into $k$ blocks, and $B(k,j)$ the same sum over admissible cuts of lines $j+1,\dots,N$ into $K-k$ blocks. Splitting off the last, respectively first, block gives $A(k,j)=\sum_{i=\max(0,j-L)}^{j-2}A(k-1,i)\,\mathrm e^{S_\omega(i+1,j)}$ with $A(0,0)=1$ and $A(0,j)=0$ for $j>0$, and $B(k,j)=\sum_{j'=j+2}^{\min(N,j+L)}\mathrm e^{S_\omega(j+1,j')}B(k+1,j')$ with $B(K,N)=1$ and $B(K,j)=0$ for $j<N$: Eq.~(5) of the main text with $(\max,+)$ replaced by $(+,\times)$. An admissible cut with $j_k=j$ is the concatenation of a $k$-block cut of $1,\dots,j$ and a $(K-k)$-block cut of $j+1,\dots,N$, and its weight is the product of theirs, so $\Pr[j_k=j\mid x]=A(k,j)B(k,j)/A(K,N)$. Each table has $(K+1)(N+1)$ entries filled by $O(L)$ terms, and the block scores are those of Lemma~1, so the time is that of Lemma~1. At equal weights the posterior odds of two cuts are $\mathrm e^{\omega(\phi(\mathbf j)-\phi(\mathbf j'))}$. (iii)~Let $j$ be a body line of the true cut and $[i,j]$ a candidate block. Under $\mathcal M$ the readings of line $j$ are drawn independently of those of lines $1,\dots,j-1$: a body line's readings are independent draws, and a line before $j$ depends at most on earlier lines of its own block. Hence $\Pr[x_{jc}=\sum_{r=i}^{j-1}x_{rc}\mid x_1,\dots,x_{j-1}]=f_c\big(\sum_{r=i}^{j-1}x_{rc}\big)\le\max_uf_c(u)$. The columns are independent processes under $\mathcal M$, every draw being independent across columns and a total depending only on its own column, so the $C$ footing events are independent.
\end{IEEEproof}
The identity rests on a wrong total being drawn from the same law $f_c$ as a constituent figure. If wrong totals follow another law $g_c$, the same computation gives
\begin{multline*}
 \log\Pr[x\mid\mathbf j]=\log\Phi(x)+K\textstyle\sum_c\log\varepsilon_c\\
 +\sum_{k,c}\Big(\log\frac{g_c(x_{j_kc})}{f_c(x_{j_kc})}+\omega^g_c(x_{j_kc})\,\mathbf 1\{B_k\text{ foots in }c\}\Big),
\end{multline*}
with $\omega^g_c(u)=\log\big(1+(1-\varepsilon_c)/(\varepsilon_cg_c(u))\big)$; the first term in the sum depends on which lines the cut makes totals, so the footing score is then the maximum-likelihood objective only approximately. Next, let some readings be $\bot$, with $f_c(\bot)>0$, and let the total of a block whose body column holds a $\bot$ be drawn from $f_c$. With $r_{kc}=1$ when the body column of $B_k$ is fully read in column $c$ and $r_{kc}=0$ otherwise, the same computation gives
\[
 \sum_{k,c}r_{kc}\Big(\log\varepsilon_c+\omega_c(x_{j_kc})\,\mathbf 1\{B_k\text{ foots in }c\}\Big)
\]
in place of the sum in Eq.~\eqref{supp:eq:loglik}, with $\mathrm{const}(x)=\log\Phi(x)$ since the terms $\log\varepsilon_c$ now sit inside the sum, and a total read as $\bot$ counting as not footing. A fully read body column that fails to foot thus counts against a cut by $\log\varepsilon_c$, where Eq.~(4) of the main text scores it as it scores a column with an unread cell. The two agree whenever, in each column, every cut has the same number of fully read body columns (or the $\varepsilon_c$ are equal and every cut has the same total number), in particular when no reading is $\bot$.

\begin{proposition}[Size bias]
\label{supp:prop:sizebias}
Let columns $c=1,\dots,C$ have lengths $R_c$ and fail independently with probability $p(R_c)$. The expected share of cells in failed columns is the size-biased mean $\Lambda=\sum_cR_c\,p(R_c)/\sum_cR_c$; whenever $p$ is nondecreasing, $\Lambda\ge\frac1C\sum_cp(R_c)$, the expected share of columns that fail. For $p(R)=1-(1-\theta)^R$, $\Lambda\le\theta\bar R_{\mathrm{sb}}$, with $\bar R_{\mathrm{sb}}=\sum_cR_c^2/\sum_cR_c\ge\bar R=\frac1C\sum_cR_c$, and $\Lambda\approx\theta\bar R_{\mathrm{sb}}$ when $\theta R_{\max}\ll1$.
\end{proposition}
\begin{IEEEproof}
A uniformly random cell lies in a column of length $r$ with probability proportional to $r$ times the number of such columns, which gives $\Lambda$. For the comparison with the column share,
\[
 \Lambda-\frac1C\sum_cp(R_c)=\frac{C\sum_cR_cp(R_c)-\sum_cR_c\sum_cp(R_c)}{C\sum_cR_c},
\]
and the numerator is $\frac12\sum_{c,c'}(R_c-R_{c'})(p(R_c)-p(R_{c'}))\ge0$ for nondecreasing $p$ (Chebyshev's sum inequality), with equality only when all lengths coincide or $p$ is constant on them; summing $1-(1-\theta)^{R_c}\le\theta R_c$ over columns also bounds the expected share of failed columns by $\theta\bar R$. The bound $1-(1-\theta)^R\le\theta R$ is Bernoulli's inequality, with equality to first order when $\theta R$ is small; $\sum R_c^2\ge(\sum R_c)^2/C$ is Cauchy--Schwarz.
\end{IEEEproof}

\emph{The protocol of Theorem~3.} The cells fill whole top-level blocks; level-$k$ blocks hold $M_k$ cells, $M_1\mid M_2\mid\dots\mid M_h$, each level-$(k+1)$ block being the union of $M_{k+1}/M_k$ level-$k$ blocks, and each block prints a total, read correctly, with consistent arithmetic. Every top-level total is corroborated; when a level-$(k+1)$ block fails, the totals of its level-$k$ blocks are corroborated, $1\le k<h$; every cell of a failing level-1 block is reviewed or withheld. A corroborated total is one unit of work and a cell reviewed or withheld is $\gamma$. A block fails exactly when it holds a defect; with real sums, a block whose errors cancel passes, its sub-blocks are not visited, and the work is no greater on any outcome.

\begin{theorem}[Nested printed sums]
\label{supp:thm:nested}
With $h$ levels of blocks of $M_1\mid M_2\mid\dots\mid M_h$ cells and $p(M)=1-(1-\theta)^M$, the expected work per cell is
\begin{equation}
 E=\frac1{M_h}+\sum_{k=1}^{h-1}\frac{p(M_{k+1})}{M_k}+\gamma\,p(M_1),
 \label{supp:eq:nested}
\end{equation}
and $(1-\theta)^{M_h}B\le E\le B$, where $B=1/M_h+\theta\sum_{k<h}M_{k+1}/M_k+\gamma\theta M_1\ge B_h^\star=(h+1)(\gamma\theta^h)^{1/(h+1)}$, with equality at the real design $M_k=(\gamma/\theta)^{k/(h+1)}/\gamma$; the least $E$ over integer designs is $B_h^\star(1+o(1))$ as $\theta\to0$. At $\gamma=1$, blocks of $3,9,27,\dots$ cells over $\lceil\log_3\theta^{-1}\rceil$ levels take $E\le\theta(4+\frac{3}{\ln3}\ln\theta^{-1})$, and any protocol that finds the defective cells from pass--fail outcomes takes on average at least $H_2(\theta)=-\theta\log_2\theta-(1-\theta)\log_2(1-\theta)$ steps per cell.
\end{theorem}
\begin{IEEEproof}
\emph{The formula.} Each cell's top-level block contributes one corroborated total per $M_h$ cells. A level-$k$ total, $k<h$, is corroborated exactly when the protocol reaches its parent, the level-$(k+1)$ block containing it, and finds it failing; a block holding a defect fails and so do all its ancestors, which hold the same defect, so this happens exactly when the parent holds a defect, with probability $p(M_{k+1})$. There are $1/M_k$ level-$k$ totals per cell, which gives the middle terms; a cell is reviewed exactly when its level-1 block holds a defect, with probability $p(M_1)$, at work $\gamma$.

\emph{The two-sided bound.} For $1\le M\le M_h$, keeping only the event of exactly one defect and then using Bernoulli's inequality, $\theta M(1-\theta)^{M_h}\le\theta M(1-\theta)^{M-1}\le p(M)\le\theta M$. Applied to every $p$ in Eq.~\eqref{supp:eq:nested}, with $1/M_h\ge(1-\theta)^{M_h}/M_h$, this gives $(1-\theta)^{M_h}B\le E\le B$.

\emph{The optimum of $B$.} $B$ is the sum of the $h+1$ positive terms $1/M_h$, $\theta M_{k+1}/M_k$ for $1\le k<h$, and $\gamma\theta M_1$, whose product telescopes to $\gamma\theta^h$. By the inequality of arithmetic and geometric means $B\ge(h+1)(\gamma\theta^h)^{1/(h+1)}=B_h^\star$, with equality exactly when every term equals $\tau=(\gamma\theta^h)^{1/(h+1)}$, that is, $M_1=\tau/(\gamma\theta)$, $M_{k+1}/M_k=\tau/\theta$, and $M_h=1/\tau$, which is $M_k=(\tau/\theta)^k/\gamma=(\gamma/\theta)^{k/(h+1)}/\gamma$; the condition on $M_h$ is then implied, since $\tau^{h+1}=\gamma\theta^h$. These sizes are real numbers, increasing in $k$ when $\gamma>\theta$, with $M_1\ge1$ exactly when $\gamma^h\theta\le1$.

\emph{Integer designs, $\theta\to0$.} Write $E_h^\star(\theta)$ for the infimum of $E$ over positive integers $M_1\mid\dots\mid M_h$, and $E_0^\star=\gamma$ (no printed sum: every cell reviewed). Upper bound: with $\nu=(\gamma/\theta)^{1/(h+1)}\to\infty$, take $M_1=\lceil\nu/\gamma\rceil$ and $M_{k+1}=M_k\lceil\nu\rceil$. Since $\gamma\nu^{-h}=\theta\nu=\tau$, each of the $h+1$ terms of $B$ is $\tau(1+o(1))$, so $E_h^\star\le B=B_h^\star(1+o(1))$. Lower bound, by induction on $h$: fix $\delta>0$ and an integer design. If $\theta M_h\le\delta$, then $E\ge(1-\theta)^{M_h}B_h^\star\ge(1-\theta)^{\delta/\theta}B_h^\star$, and $(1-\theta)^{\delta/\theta}\to\mathrm e^{-\delta}$. If $\theta M_h>\delta$, then $p(M_h)\ge1-\mathrm e^{-\delta}$, and dropping $1/M_h$ and using $p\le1$ in the remaining terms,
\begin{align*}
 E&\ge p(M_h)\Big(\frac1{M_{h-1}}+\sum_{k=1}^{h-2}\frac{p(M_{k+1})}{M_k}+\gamma p(M_1)\Big)\\
 &\ge(1-\mathrm e^{-\delta})E_{h-1}^\star,
\end{align*}
the bracket being the work of the $(h-1)$-level design $M_1\mid\dots\mid M_{h-1}$ (for $h=1$, read $E\ge\gamma p(M_1)\ge(1-\mathrm e^{-\delta})E_0^\star$). By induction $E_{h-1}^\star\ge(1+o(1))B_{h-1}^\star$ for $h\ge2$, and $B_{h-1}^\star/B_h^\star$ is a constant times $\theta^{-1/(h(h+1))}\to\infty$ (for $h=1$, $E_0^\star/B_1^\star=\gamma/(2\sqrt{\gamma\theta})\to\infty$), so for small $\theta$ such designs do worse than $B_h^\star$. Hence $\liminf E_h^\star/B_h^\star\ge\mathrm e^{-\delta}$ for every $\delta>0$, and with the upper bound $E_h^\star/B_h^\star\to1$.

\emph{Powers of three.} With $M_k=3^k$ and $\gamma=1$, $B=3^{-h}+3(h-1)\theta+3\theta=3^{-h}+3h\theta$. For $h=\lceil\log_3\theta^{-1}\rceil$, $3^{-h}\le\theta$ and $h<\ln\theta^{-1}/\ln3+1$, so $E\le B\le\theta(4+\frac3{\ln3}\ln\theta^{-1})$.

\emph{The information bound.} Under the idealization every step of any such protocol, a corroborated total compared with its block's sum or a cell reviewed, has two outcomes, and the protocol ends knowing the set $\mathcal D$ of defective cells, which the outcomes determine. Because the protocol chooses each step from the outcomes so far and decides from them when to stop, the outcome sequences of distinct sets $\mathcal D$ form a prefix-free binary code for $\mathcal D$, and by the noiseless coding theorem the expected number of steps is at least the entropy of $\mathcal D$, $R\,H_2(\theta)$ bits for $R$ cells defective independently with probability $\theta$. At $\gamma=1$ every step is one unit of work. Since $H_2(\theta)\ge\theta\log_2\theta^{-1}$, the powers of three come within a factor $\big(4+\frac3{\ln3}\ln\theta^{-1}\big)\ln2/\ln\theta^{-1}\to3\ln2/\ln3\approx1.89$ of the bound.
\end{IEEEproof}

\begin{theorem}[Segmentation recovery]
\label{supp:thm:segment}
Let $\mathbf j^\star$ be the true cut of Section~IV-B of the main text, with true blocks $B_1,\dots,B_K$ ending at the true total lines $j^\star_1<\dots<j^\star_K$ and footing counts $\phi_1,\dots,\phi_K$; call a candidate block coincidental if its last line is not a true total line and aligned if it ends at some $j^\star_k$ without being $B_k$. Suppose (i) every true block has length in $[2,L]$ and foots in at least $\phi_{\min}$ columns; (ii) every coincidental block foots in at most $\phi_0\le\phi_{\min}-1$ columns; and (iii) every aligned block ending at $j^\star_k$ foots in at most $\phi_k$ columns. Then the recursion of Eq.~(5) of the main text with $\lambda=0$ returns $\mathbf j^\star$, its score exceeding every other cut's by at least the number of that cut's coincidental blocks, hence by at least one, and it still returns $\mathbf j^\star$ when each of the nonnegative cue weights is smaller than $1/(3K)$. If, moreover, a coincidental block foots in each of $C'$ columns independently with probability at most $\varpi$, as it does under the model of Theorem~2 with $\varpi=\max_c\max_uf_c(u)$, then for every fixed $m$ the probability that some coincidental block foots in $m$ or more of those columns is at most $NL\,(C'\varpi)^{m}/m!$, which with $m=\phi_0+1$ bounds the chance that hypothesis~(ii) fails; a column in which every block foots, an all-zero column, adds $K$ to every $K$-cut and can be left out of the count.
\end{theorem}
\begin{IEEEproof}
In a column where $B_k$ foots, an aligned block foots only if the lines by which it differs from $B_k$ are all zero there. Let $\mathbf j\ne\mathbf j^\star$ be a cut and $A\subseteq\{1,\dots,K\}$ the set of $k$ such that some block of $\mathbf j$ ends at $j^\star_k$; these blocks are true or aligned, so by (iii) each scores at most $\phi_k$, and the remaining $K-|A|$ blocks of $\mathbf j$ are coincidental and score at most $\phi_0$ by (ii). Hence $\mathrm{score}(\mathbf j)\le\sum_{k\in A}\phi_k+(K-|A|)\phi_0$, while $\mathrm{score}(\mathbf j^\star)=\sum_k\phi_k\ge\sum_{k\in A}\phi_k+(K-|A|)\phi_{\min}$ by (i), so the gap is at least $(K-|A|)(\phi_{\min}-\phi_0)\ge K-|A|$. If $K-|A|=0$, the $K$ cut points of $\mathbf j$ are the $K$ true total lines and $\mathbf j=\mathbf j^\star$; so the gap is at least one for every other cut. Three nonnegative cues per block add less than $3K\cdot1/(3K)=1$ to any cut and at least zero to $\mathbf j^\star$, which cannot reverse a gap of one. For the probabilistic statement, there are at most $NL$ candidate blocks, and one foots in $m$ or more of $C'$ independent columns with probability at most $\binom{C'}{m}\varpi^{m}\le(C'\varpi)^{m}/m!$, by the union bound over the $m$-sets of columns; apply the union bound again over blocks. The threshold $m$ must be fixed in advance, since the realized $\phi_{\min}$ is random. A column in which every block foots adds $K$ to every $K$-cut and changes no comparison.
\end{IEEEproof}

\begin{lemma}[The likelihood of Section~VI-C]
\label{supp:lem:mle}
For columns of lengths $R_c$, each failing the gate with probability $1-(1-\theta)^{R_c}$ independently, the log-likelihood $\ell(\theta)$ of Section~VI-C of the main text is strictly concave in $\lambda=-\log(1-\theta)$ when some column fails; when some column also passes, it has a unique maximizer $\hat\theta\in(0,1)$, and each set $\{\theta:\ell(\theta)\ge\ell(\hat\theta)-z\}$ is an interval.
\end{lemma}
\begin{IEEEproof}
In $\lambda\in(0,\infty)$, $\ell=\sum_{\mathrm{fails}}\log(1-\mathrm e^{-\lambda R_c})-\lambda\sum_{\mathrm{passes}}R_c$. For $a>0$, $g(\lambda)=\log(1-\mathrm e^{-a\lambda})$ has $g''=-a^2\mathrm e^{a\lambda}/(\mathrm e^{a\lambda}-1)^2<0$, and the second sum is linear, so $\ell$ is strictly concave when some column fails. As $\lambda\to0$ a failing column's term tends to $-\infty$, and as $\lambda\to\infty$ a passing column's term does, so a unique maximizer exists, and the increasing map $\lambda\mapsto1-\mathrm e^{-\lambda}$ carries the superlevel sets, intervals, to intervals.
\end{IEEEproof}
\end{document}